\documentclass[12pt,a4paper]{article}

\usepackage[latin1]{inputenc}
\usepackage[english]{babel}
\usepackage{amsmath}
\usepackage{color}
\usepackage{amsfonts,enumitem}
\usepackage{amssymb}
\usepackage{hyperref} 
\usepackage{graphicx,tikz}

\usepackage{algorithmic}
\usepackage[ruled,noend]{algorithm2e}
\usepackage[skins]{tcolorbox}

\newcommand{\conv}{\mathrm{conv}}

\newcommand{\RR}{\mathbb{R}}
\newcommand{\NN}{\mathbb{N}}

\newcommand{\partialc}{\partial^c}

\newtheorem{theorem}{Theorem}
\newtheorem{lemma}{Lemma}
\newtheorem{proposition}{Proposition}
\newtheorem{corollary}{Corollary}

\newtheorem{definition}{Definition}

\newtheorem{remark}{Remark}

\newenvironment{proof}[1][]{\noindent {\bf Proof #1:\;}}{\hfill $\Box$}

\newcommand{\jac}{\mathrm{Jac}\,}

\newcommand{\R}{\mathbb{R}}

\makeatletter
\newcommand{\customlabel}[2]{%
   \protected@write \@auxout {}{\string \newlabel {#1}{{#2}{\thepage}{#2}{#1}{}} }%
   \hypertarget{#1}{#2}
}
\makeatother

\title{Path differentiability of ODE flows}
\begin{document}
\author{
Swann Marx$^{1}$ and Edouard Pauwels$^{2}$		 
}

\footnotetext[1]{CNRS UMR 6004, LS2N, \'Ecole Centrale de Nantes, F-44000 Nantes, France.}
\footnotetext[2]{IRIT, Universit\'e de Toulouse, CNRS, 118 route de Narbonne, F-31400 Toulouse, France.}

\date{}

\maketitle

\begin{abstract}
    We consider flows of ordinary differential equations (ODEs) driven by path differentiable vector fields. Path differentiable functions constitute a proper subclass of Lipschitz functions which admit conservative gradients, a notion of generalized derivative compatible with basic calculus rules. Our main result states that such flows inherit the path differentiability property of the driving vector field. We show indeed that forward propagation of derivatives given by the sensitivity differential inclusions provide a conservative Jacobian for the flow. This allows to propose a nonsmooth version of the adjoint method, which can be applied to integral costs under an ODE constraint. This result constitutes a theoretical ground to the application of small step first order methods to solve a broad class of nonsmooth optimization problems with parametrized ODE constraints. This is illustrated with the convergence of small step first order methods based on the proposed nonsmooth adjoint.
\end{abstract}

\section{Introduction}

\subsection{General context}
We consider the ordinary differential equation (ODE for short), for some $T > 0$
\begin{align}
				\dot{X}(t) &= F(X(t)), \qquad \forall t \in [0,T]
				\label{eq:mainODE}\\
				X(0) &= x, \nonumber
\end{align}
where $F \colon \RR^p \rightarrow \RR^p$ is a Lipschitz function and $x \in \RR^p$. We denote by $\phi\colon \RR^p \times [0,T] \to \RR^p$ the corresponding flow which associates to $(x,t) \in \RR^p \times [0,T]$ the value $X(t)$ where $X\colon \RR \to \RR^p$ is the solution to \eqref{eq:mainODE} with $X(0) = x$. The flow $\phi$ typically inherits the regularity of $F$. For example if $F$ is $C^1$, then $\phi$ is also $C^1$ (see \textit{e.g.} \cite[Section 17.6]{hirsch2012differential}). 

In our setting, the flow $\phi$ is Lipschitz (see \textit{e.g.}, \cite[Section 17.4]{hirsch2012differential}). A notion of generalized derivative adapted to Lipschitz function is due to Clarke. The Clarke Jacobian takes values in subsets of $\RR^{p \times p}$. We denote by $J_F \colon \RR^p \rightrightarrows \RR^{p \times p}$ the Clarke Jacobian of $F$ \cite[Section 2.6]{clarke1983optimization}. For $x\in \RR^p$, it is defined as follows
\begin{align*}
     J^c_F(x) = \mathrm{conv} \left\{ v \in \RR^{p \times p},\, \{x_k\}_{k \in \NN} \subset R,\, x_k \to x,\, \mathrm{Jac}_F(x_k) \to v, k \to \infty \right\},
\end{align*}
where $R$ is any full measure set where $F$ is differentiable and $\mathrm{Jac}_F$ is the usual Jacobian of $F$. Our main question of interest is to obtain generalized derivatives of $\phi$ from the knowledge of Clarke Jacobian of $F$. This question requires to take a more detailed look at the regularity of $F$.

\subsection{Path differentiability of the flow}

The class of Lipschitz functions $F$ is too large for our purpose. Indeed, for generic Lipschitz $F$ the Clarke Jacobian, $\partialc F$ carries no information about the function itself  \cite{wang1995pathological,borwein2001generalized,borwein2017genralisations}. In particular, it is proved in \cite{borwein2001generalized} that generic $1$-Lipschitz functions have the same constant subgradient. 

Therefore, to obtain meaningful calculus rules, we need to restrict $F$ to be in a well behaved subclass. We choose the class of path differentiable functions, which was identified by several authors to be well behaved from a nonsmooth analysis perspective \cite{valadier1989entrainement,borwein1998chain,bolte2020conservative}. 
Let us emphasize that, although this is a negligible subclass of Lischitz functions, it is ubiquitous in potential applications as all semi-algebraic functions (more generally definable functions) are path differentiable \cite{bolte2020conservative}. This encompasses virtually any function Lipschitz $F$ which can be written using an elementary logical formula involving elementary real operations including powers, exponential, logarithms, quotients, including large classes of numerical programs \cite{bolte2020mathematical}.

Following \cite{bolte2020conservative}, $F$ is called path differentiable, if it satisfies a chain rule along absolutely continuous curves: for any absolutely continuous $\gamma \colon [0,1] \to \RR^p$, we have for almost all $t \in [0,1]$,
\begin{equation*}
\frac{d}{dt} F(\gamma(t)) = D \dot{\gamma}(t),\qquad \forall D\in J^c_F(\gamma(t)) \subset \RR^{p\times p}.
\end{equation*}
In a first step, we will be interested in the following question regarding regularity of $\phi$
\begin{center}
    Does path differentiability  of $F$ imply path differentiability of $\phi$?
\end{center}
We provide a positive answer to this question. The result is stated in Corollary \ref{cor:flowPathDifferentiable}. The proof is based on a differential inclusion which generalizes the variational equation for smooth ODEs (see for example \cite[Section 17.6]{hirsch2012differential}) to the Lipschitz vector field $F$.
This variational inclusion is described in \cite[Section 7.4]{clarke1983optimization}. For any $x \in \RR^p$, the latter is defined by the differential inclusion
\begin{align}
				\dot{V}(t) &\in J^c_F(\phi(x,t)) V(t) \nonumber, \text{ for almost all } t \in [0,T]\\
				V(0) &= I \in \RR^{p \times p}
				\label{eq:sensitivityDI}.
\end{align}
where $V$ is to be found among absolutely continuous functions from $[0,T]$ to $\RR^{p \times p}$. Equation \eqref{eq:sensitivityDI} can be seen as a formal differentiation of equation \eqref{eq:mainODE}. As proved in \cite[Theorem 7.4.1]{clarke1983optimization}, the Clarke Jacobian of the flow $J^c_\phi$ is to be found among the solutions of \eqref{eq:sensitivityDI}. More precisely, denoting by $\psi$ the function $x \mapsto \phi(x,T)$, we have 
$$J^c_\psi(x) \subset U(x) := \{V(T),\, V \text{ solution of \eqref{eq:sensitivityDI} }\}$$ 
for all $x \in \RR^p$. However, as shown in \cite[Example 3.8]{barton2018computationally}, this inclusion can be strict even for a relatively simple $F$ in $\RR^2$ (see Section \ref{sec:counterEx}). Following \cite{bolte2020conservative}, path differentiability of the flow is characterized by existence of a conservative Jacobian for $\psi$. More precisely, we show that the set valued map $U$, despite not being necessarily equal to the Clarke Jacobian of $\psi$, still satisfies the chain rule: for any absolutely continuous $\gamma \colon [0,1] \to \RR^p$, we have for almost all $t \in [0,1]$,
\begin{equation*}
\frac{d}{dt} \psi(\gamma(t)) = D \dot{\gamma}(t),\qquad \forall D\in   U(\gamma(t)) \subset \RR^{p\times p}.
\end{equation*}
This is the result given in Theorem \ref{th:flowPathDifferentiable}. Thanks to this chain rule property, $U$ is a conservative Jacobian of $\psi$, characterized as the set of solutions to a sensitivity differential inclusion. The existence of such a conservative Jacobian implies that $\psi$ inherits the path-differentiable regularity of $F$.

The mapping $U$ being a conservative Jacobian of $\psi$ has several consequences for the flow. For example, as given in \cite[Corollary 5]{bolte2020conservative}, we have for Lebesgue almost all $x \in \RR^p$
\begin{align*}
    U(x) = \{ \jac \psi (x)\},
\end{align*}
which means that the differential inclusion \eqref{eq:sensitivityDI} provides a unique matrix that is the (classical) Jacobian of $\psi$. This allows to draw a connection with more classical notions of generalized derivatives. For example, using \cite[Theorem 6.5]{evans2015measure}, the mapping $U$ (restricted to the set where it is a singleton) can be interpreted as a weak derivative of the flow $\psi$ in the sense of Sobolev spaces.

\subsection{Optimizing integral costs under ODE constraints}

First motivations to address path differentiability of the flow relates to optimization problems of the form
\begin{align}
    \label{eq:integralintro}
	\min_{\theta \in \RR^m}\quad L(\theta):= \quad  &\int_{t=0}^{t=T} \ell(Z(t))dt + \ell_T(Z(T)),\\
	\mathrm{where} \quad & Z(0) = \bar{z} \nonumber\\
	& \dot{Z}(t) = H(Z(t), \theta),\quad \forall t \in [0,T]\nonumber
\end{align}
where $\ell:\mathbb{R}^p\rightarrow \mathbb{R}$, $\ell_T:\mathbb{R}^p\rightarrow \mathbb{R}$ and $H \colon \RR^{p} \times \RR^m \to \RR^p$ are Lipschitz, path differentiable functions and $\bar{z} \in \RR^p$ is fixed. Note moreover that the flow depends on a given parameter $\theta\in\mathbb{R}^m$. The decision variable in problem \eqref{eq:integralintro} is a parameter vector $\theta$. Such an optimization problem appears in many applications such as machine learning \cite{chen2018neural}, data assimilation \cite{lewis2006dynamic} or geophysics \cite{plessix2006review}.  

We will consider first order methods of gradient type to tackle problem \eqref{eq:integralintro} algorithmically. These methods generate sequences by recursively following negative gradient directions. The function $L$ is Lipschitz and possibly nonsmooth so we need to use a generalized notion of gradient. The integral part of the loss $L$ consists in a composition of the flow and a Lipschitz integral cost. However, Clarke Jacobian of the flow may be strictly contained in solutions of the variational inclusion \eqref{eq:sensitivityDI}, see  \cite[Example 3.8]{barton2018computationally}. Fortunately, conservative gradients can be used in place of usual gradients in a nonsmooth optimization context, provided that the objective function is path differentiable \cite{bolte2020conservative,bolte2020long}. Therefore, the main questions we need to adress are the following:
\begin{center}
    Is the loss $L$ path differentiable? How to obtain a conservative gradient for $L$?
\end{center}

We leverage our main result on path differentiability of the flow, and the compatibility of conservative Jacobian with calculus rules to show that $L$ is indeed path differentiable. More precisely, we show that a formal differentiation of $L$ (application of integral differentiation rules which hold in the smooth case) using solutions of the variational inclusion \eqref{eq:sensitivityDI} provides a conservative gradient for $L$, this is described in Corollary \ref{cor:conservativeJacIntParam}. 

Numerical computation of a solution of the variational inclusion \eqref{eq:sensitivityDI}, for example using Euler discretization \cite{dontchev1992difference}, requires to solve a differential inclusion of size $p \times p$. In the context of smooth ODEs, it is known that the size of the system to be solved can be reduced to $p$ by using the adjoint method (see e.g., \cite{cao2003adjoint}) at the cost of solving an ODE backward in time. We derive a nonsmooth counterpart of the adjoint system using the conservative Jacobian framework and show that solutions to the adjoint system are elements of the conservative gradient for the loss $L$ given in Corollary \ref{cor:conservativeJacIntParam}. This is described in Corollary \ref{cor:ajointParam}. 

Application of known results in nonsmooth optimization \cite{bolte2020long} show that using the prosposed conservative gradient in place of a gradient in a small step first order method context induces a minimizing behavior and generates sequences attracted by sets defined by an optimality condition. In other word, the output given by the proposed adjoint methods may be used as a first order optimization oracle to implement gradient type methods for the problem \eqref{eq:integralintro}. This result is formally described in Corollary \ref{cor:convergenceSmallStep}.

\subsection{Related work}
Combination of adjoint differentiation and small step methods of gradient type is at the heart of numerical methods for training neural ordinary differential equations models \cite{chen2018neural,dupont2019augmented}. Our results provide a theoretical ground for these approaches for which dedicated numerical librairies exist and are broadly used, such as \texttt{torchdiffeq} in \texttt{python}. These constitute one of the motivations for our investigation. 

The use of Clarke's generalized derivatives in a dynamical systems context has been at the heart of nonsmooth analysis developments \cite{clarke1983optimization}, in variational analysis \cite{clarke1975euler} and stability analysis \cite{clarke1998asymptotic,acary2008numerical}. More recent contributions include existence and Lipschitz regularity of nonsmooth differential algebraic equations \cite{stechlinski2017dependence} and generalizations in Wasserstein space \cite{bonnet2021differential}. 

The variational inclusion dates back to the work of Clarke \cite{clarke1983optimization}. Providing meaning to this equation has been an active topic of research. Let us mention the work of \cite{pang2009solution} which prove semismoothness of the flow induced by semismooth gradient fields. In this case the variational inclusion becomes an equation and allows to obtain directional derivatives. This result was extended by
\cite{khan2014generalized} to handle possibly discontinuous time dependency and lexicographic derivatives \cite{nesterov2005lexicographic}. Deducing lexicographic derivatives from variational equation was extended to differential algebraic equations in \cite{stechlinski2016generalized}. All these works are centered around notions of directional derivatives and forward derivative propagation. We are not aware of further interpretations of the variational inclusion \eqref{eq:sensitivityDI} beyond directional derivatives and forward propagation. In an optimization context, directional derivatives are not sufficient as one needs to find candidate descent directions. This constitutes another motivation for the proposed developments.

\subsection{Organization}

The paper is organized as follows. Section \ref{sec_presentation} provides notations, definitions and details about the example of Clarke Jacobian forward propagation failure in \cite[Example 3.8]{barton2018computationally}. Section \ref{sec_preliminary} contains preliminary results with their proofs. Section \ref{sec_flow} is devoted to the first main result, the flow of \eqref{eq:mainODE} inherits path differentiability of $F$. Section \ref{sec_adjoint} shows that integral costs in optimization with ODE constraints are path differentiable as soon as the the loss function is path differentiable. It is also proved that the adjoint method can be applied in this context to estimate elements of the corresponding conservative gradient. Section \ref{sec_neural} is devoted to an extension of these results, from initial conditions dependency to the more general parametric case described in \eqref{eq:integralintro}. The latter includes also convergence guaranties for the small step gradient like method. Some concluding remarks are collected in Section \ref{sec_conclusion} together with further research lines. Finally, Appendix \ref{sec_appendix} gathers technical results used throughout the paper.

\section{Notation and definitions}

\label{sec_presentation}

\paragraph{Notation.} 
Set $\RR_+ = [0,\infty)$. Given $p\in\mathbb{N}$, we will denote $\Vert \cdot\Vert$ the norm and $\langle \cdot,\cdot\rangle$ the scalar product in $\mathbb{R}^p$. We will denote by $\Vert \cdot \Vert_{op}$ the operator norm for matrices, i.e. if $A\in\mathbb{R}^{p\times p}$, then $\Vert A\Vert_{op}:= \sup_{\Vert v\Vert\leq 1} \Vert Av\Vert$. The Frobenius norm is defined and denoted by $\Vert A\Vert_F:=\sqrt{\mathrm{Tr}(A^\top A)}$, where $A\in\mathbb{R}^{p\times p}$, $\mathrm{Tr}$ is the trace operator and $A^\top$ is the transpose of $A$. The supremum norm is denoted by $\Vert \cdot \Vert_{\infty}$.

We recall that, due to Rademacher theorem \cite[Theorem 3.1]{evans2015measure}, any locally Lipschitz function is almost everywhere differentiable. \emph{Absolutely continuous} curves $\gamma:\mathbb{R}\rightarrow \mathbb{R}^p$ are functions admitting a Lebesgue integrable derivative (defined for almost all $t\in \mathbb{R}$), such that for any $t\geq 0$: 
$$
\gamma(t) - \gamma(0) = \int_0^t \dot{\gamma}(s) ds.
$$
Given three metric spaces $H$, $S$ and $Y$, a \emph{Carath\'edory function} $f \colon (x,t)\in H\times S \mapsto f(x,t) \in Y$ is a function such that $x\mapsto f(x,t)$ is Borel measurable for each $t\in S$ and such that $t\mapsto f(x,t)$ is continuous for each $x\in H$. We say that a function $f:\mathbb{R}^p\rightarrow \mathbb{R}$ is \emph{lower semi-continuous}, if for every every sequence $(x_k)_{k\in\mathbb{N}}\subset \mathbb{R}^p$ such that $\lim_{k\rightarrow \infty} x_k = \bar x$, one has $f(\bar{x}) \leq \liminf_{k\rightarrow + \infty} f(x_k)$.  

A set valued map $D:\mathbb{R}^p\rightrightarrows \mathbb{R}^q$ is a function from $\mathbb{R}^p$ to a subset of $\mathbb{R}^q$. We say that $D$ has a \emph{closed graph} if, for any convergent sequences $(x_k)_{k\in\mathbb{N}}\subset \mathbb{R}^p$ and $(v_k)_{k\in\mathbb{N}}\subset \mathbb{R}^q$, with $v_k\in D(x_k)$, one has $\lim_{k\rightarrow \infty} v_k \in D(\lim_{k\rightarrow + \infty} x_k)$.

\paragraph{Path differentiablity and conservative Jacobians.}

The notion of path differentiable functions has been introduced in \cite{bolte2020conservative}, this class of regularity allows to apply basic differential calculus rules such as the chain rule. As explained in \cite{bolte2020conservative}, the notion of \emph{conservativity}, defined just below, is crucial to define path differentiable functions.

\begin{definition}[Conservative Jacobian]
Let $D:\:\mathbb{R}^p\rightrightarrows \RR^{n \times p}$ be a locally bounded, graph closed, nonempty valued map and $f:\mathbb{R}^p\rightarrow \RR^n$ be a locally Lipschitz continuous function. Then, $D$ is said to be a conservative Jacobian of $f$ if and only if, for any absolutely continuous curve $\gamma: [0,1] \rightarrow \RR^p$, the function $t\mapsto f(\gamma(t))$ satisfies, for almost all $t\in [0,1]$
\begin{equation*}
\frac{d}{dt} f(\gamma(t)) = V\dot{\gamma}(t),\: \forall V\in D(\gamma(t)).
\end{equation*}

Equivalently, $D$ is a conservative Jacobian of $f$ if and only if, for any measurable selection $V(t) \in D(\gamma(t))$ for all $t\in [0,1]$,:
\begin{equation*}
    f(\gamma(1)) - f(\gamma(0)) = \int_0^1  V(t)\dot{\gamma}(t) dt.
\end{equation*}
When $n=1$, we say that $D$ is a conservative gradient.
\end{definition}

Conservatives gradients are defined in the same way for real valued functions (see \cite{bolte2020conservative}). Throughout the paper, we require conservative gradients and Jacobians to be convex. This is not too restrictive due to the following remark.
\begin{remark}
        It follows from the definition that if $D$ is conservative, then its pointwise convex hull $x \rightrightarrows \conv\{D(x)\}$ is also conservative \cite{bolte2020conservative}. 
\end{remark} 
Conservativity leads to the notion of path differentiability:

\begin{definition}[Path differentiable function]
We say that $f:\RR^p \rightarrow \RR^n$ is path differentiable if there exists a set valued map $D$ such that $D$ is a conservative Jacobian for $f$. 
\end{definition}

\begin{remark}
    If $J_f$ is a conservative Jacobian for $f$, then we have $J^c_f(x) \subset \conv\{J_f(x)\}$ for all $x$ \cite{bolte2020conservative}, in particular $J^c_f$ is conservative. Hence $J^c_f$ being conservative is a characterization of path differentiability of $f$ as stated in the intoduction.
\end{remark}

\paragraph{Dynamical systems.}
Consider $F \colon \RR^p \rightarrow \RR^p$ the Lipschitz function given in \eqref{eq:mainODE} that is assumed path differentiable. We denote by $J_F \colon \RR^p \rightrightarrows \RR^p$ a bounded convex valued conservative Jacobian for the vector field $F$ which appears in \eqref{eq:sensitivityDI}. Throughout the paper, we denote by $K> 0$ a bound on the operator norm of $J_F$, that is, 
\begin{align}
   \sup_{x \in \RR^p, J \in J_F(x)} \| J\|_{\mathrm{op}} \leq K.
   \label{eq:boundJacobian}
\end{align}

We introduce the following map
\begin{align}
				U \colon \RR^p \times [0,T] &\rightarrow \RR^{p \times p} \nonumber\\
				(x,t) &\rightrightarrows V(t) \qquad V \text{ solution to \eqref{eq:sensitivityDI}} 
				\label{eq:candidateConservativeField},
\end{align}
which is a candidate for being a conservative Jacobian for the flow $\phi$ of \eqref{eq:mainODE}.
The main result of this paper is to show that $(x,t) \rightrightarrows (U(x,t), F(\phi(x,t))$ is a conservative Jacobian for the flow $\phi$, where we have used matrix concatenation.

Since $F$ is assumed to be Lipschitz, the set of solutions to \eqref{eq:sensitivityDI} is composed by Lipschitz functions, as stated in the following lemma. 

\begin{lemma}
    For any $x\in\mathbb{R}^p$, $T>0$, the set of solutions to \eqref{eq:sensitivityDI} is non empty and only contains $L$-Lipschitz functions with $L = K\sqrt{p}\exp(KT)$.
    \label{lem:lipschitzSolution}
\end{lemma}
\begin{proof}
    By hypotheses on $J_F$, the solution set of \eqref{eq:sensitivityDI} is nonempty and defined on maximal intervals invoking \cite[Theorem 4, p. 101]{aubin1984differential}. 
    
    It remains to show that solutions to \eqref{eq:sensitivityDI} are bounded. Indeed, once one has a bound on $V$, one deduces a bound on $\dot{V}$ through the following inequality, which holds for a.e. $t\in [0,T]$ and is sufficient to ensure Lipschicity,
    \begin{equation}
    \label{eq:bound-derivatives}
    \Vert \dot{V}(t)\Vert_F \leq K \Vert V(t)\Vert_F.
    \end{equation}
    
    Using \eqref{eq:sensitivityDI}, \eqref{eq:boundJacobian} and \eqref{eq:bound-derivatives}, for a.e. $t\in [0,T]$, one has:
    $$
    \frac{d}{dt} \Vert V(t)\Vert_F^2 = 2 \mathrm{Tr}(V(t)^\top \dot{V}(t)) \leq 2K \Vert V(t)\Vert_F^2.
    $$
    Thanks to Lemma \ref{lem:gronwall}, and using the fact that $V(0) = I$ (see \eqref{eq:sensitivityDI}), one deduces that, for all $t\in [0,T]$:
    $$
    \Vert V(t) \Vert_F^2 \leq \Vert V(0) \Vert_F^2 \exp(2K T) = p \exp(2K T)
    $$
    This latter equation together with \eqref{eq:bound-derivatives} shows that the solutions $V$ to \eqref{eq:sensitivityDI} are $L$-Lipschitz with $L=K \sqrt{p}\exp(KT)$. This concludes the proof of the Lemma.\end{proof}

\begin{remark}[On the Lipschitz assumption]
The vector field $F$ has been supposed to be Lipschitz with a uniform bound on a conservative Jacobian in \eqref{eq:boundJacobian}, which is a stronger assumption than the (classical) local Lipschitz assumption. Under local Lipschicity, solutions only exist in a time interval which could be bounded with endpoint depending on initial condition. The global Lispchicity assumption allows to avoid such discussions, but there exist possible extensions which would allow to relax it:
\begin{itemize}
    \item[1.] If we suppose the trajectories (and the initial condition) to belong to some compact set, then local and global Lipschicity will be essentially equivalent for our purpose. For example if $F$ maximal monotone \cite[Chapter 7]{brezis2010functional}, one can show that the trajectories of \eqref{eq:mainODE} are bounded, independently of the initial condition. 
    \item[2.] A different possibility is to assume that all solutions to \eqref{eq:mainODE} are well defined on $[0,T]$ for any initial condition. One can see our global Lipschicity assumption as a sufficient condition.
\end{itemize}
\end{remark}

\begin{remark}
It is worth mentioning that the unknown of \eqref{eq:sensitivityDI} is a matrix, and not a vector as it is commonly defined in textbooks such as \cite{aubin1984differential,filippov1988differential}. It is always possible to identify a $p \times p$ matrix with a vector of dimension $p \times p$, and the meaning of matrix differential inclusion follows by using this identification.
\end{remark}

\subsection{Failure of formal differentiation with Clarke Jacobian}
\label{sec:counterEx}
Following \cite[Example 3.8]{barton2018computationally}, consider an instance of \eqref{eq:mainODE} as follows
\begin{align*}
    \begin{pmatrix}
                    \dot{X}_1 \\ \dot{X}_2
    \end{pmatrix}
    =
    \begin{pmatrix}
                    (1-X_2) |X_1|\\ 1
    \end{pmatrix},
\end{align*}
it can be proved that for any initilization $X_1(0)$ and $X_2(0) = 0$, we have $X_1(2) = X_1(0)$ and $X_2(2) = X_2(0) + 2$, therefore, the flow is differentiable at $T=2$ and its Jacobian is the identity matrix. Furthermore, if $X_1(0) = 0$, then we actually have $X_1(t) = 0$ for all $t \in \RR$. However the variational inclusion \eqref{eq:sensitivityDI} for the particular initialization $X_1(0) = X_2(0) = 0$ reads
\begin{align*}
\dot{M}    \in 
    \begin{pmatrix}
                    [- |1-t|, |1-t|] & 0 \\ 0& 0
    \end{pmatrix}
    M
\end{align*}
with $M(0) = I$ where we chose the conservative derivative of absolute value to be the usual derivative everywhere except at $0$ where it is the segment $[-1,1]$. All entries of $M$ remain constant in time, except for the first one which we denote by $m$. The two extreme solutions for $m$ are given by $\dot{m} = |1-t|m$ and $\dot{m} = - |1-t| m$ which leads to $m(2) \in [1/e, e]$. Therefore the variational inclusion fails to provide the correct subgradient for $X_1(2)$ with respect to the initial condition $X_1(0)$ (this should be $1$). This example highlights the fact that it is not possible to prove that the sensitivity analysis differential inclusion \eqref{eq:sensitivityDI} provides subgradients in general. In this example, the discrepancy occurs at the origin only, and, as described in the forthcoming results, the solutions to \eqref{eq:sensitivityDI} actually provide a conservative Jacobian for the flow. 

\section{Preliminary results}
\label{sec_preliminary}

Fix any $T >0$, we define the following mapping:
\begin{align*}
				\mathcal{U}\colon \RR^p &\rightrightarrows \mathcal{C}([0,T], \RR^{p \times p}) \\
				x& \rightrightarrows \left\{ t \mapsto V(t), \, V \text{ solution to \eqref{eq:sensitivityDI} }\right\}.
\end{align*}
We call the mapping $\mathcal{U}$ the solution mapping of the differential inclusion \eqref{eq:sensitivityDI} whose values are Lipschitz functions from $[0,T]$ to $\RR^{p \times p}$ (see Lemma \ref{lem:lipschitzSolution}). Note that, for all $t \in [0,T]$, $U(x,t) = \{V(t), \ V \in \mathcal{U}(x) \}$, where $U$ is defined in \eqref{eq:candidateConservativeField}. We introduce a Castaing representation for functions with values in Lipschitz subsets of $\mathcal{C}([0,T])$ which will allow to specify some technical measurability issues for $\mathcal{U}$.

\begin{proposition}[Castaing representation of solution mappings]
    Given $T>0$, $L > 0$, denote by $\mathcal{L}$ the space of $L$-Lipschitz functions from $[0,T]$ to $\RR^q$, endowed with the supremum norm. Consider a set-valued map $\mathcal{V}:\mathbb{R}^m\rightrightarrows \mathcal{L}$ with closed graph and non empty values. Then $\mathcal{V}$ admits a \emph{Castaing representation}, that is, a sequence of Borel measurable functions from $\RR^m$ to $\mathcal{L}$, $(M_n)_{n\in\mathbb{N}}$, such that $\mathcal{V}=\overline{\lbrace M_1(x),M_2(x),\ldots\rbrace},$ for each $x\in \mathbb{R}^m$, where the closure and Borel measurability are induced by $L^\infty$ norm over continuous functions. Furthermore, for all $i \in \NN$, $M_i$ can be seen as a function $\RR^m \times [0,T] \to \RR^q$ and we have that $(x,t) \mapsto M_i(x,t)$ is a Carath\'eodory function, $L$ Lipschitz in $t$ for fixed $x$ and Borel measurable in $x$ for fixed $t$.
    \label{prop:castaingRepresentation}
\end{proposition}
\begin{proof}
    Recall that $\sigma$-compact sets are sets defined as the union of countably many compact subspaces. Since the domain of the $\mathcal{V}$ is obviously $\sigma$-compact and takes values in the space $\mathcal{L}$ which is, by Lemma \ref{lem:lipschitzSigmaCompact}, also $\sigma$-compact, then on can invoke \cite[Theorem 18.20]{aliprantis2005infinite} to deduce that $\mathcal{V}$ is Borel measurable. Finally, using \cite[Corollary 18.14]{aliprantis2005infinite}, there exists a Castaing representation for $\mathcal{V}$. Moreover, by \cite[Theorem 4.55]{aliprantis2005infinite}, this representation is actually a sequence of Carath\'eodory functions, which is our desired result.
\end{proof}

We are now in position to state a technical representation result for the set of solutions of \eqref{eq:sensitivityDI}. Fix any $T>0$ and any absolutely continuous function $\gamma:[0,1]\rightarrow \mathbb{R}^p$. We define the following mapping:
\begin{align*}
  \mathcal{U}_\gamma\colon [0,1]&\rightrightarrows \mathcal{C}([0,T],\mathbb{R}^{p\times p})\\
  r &\rightrightarrows V\in \mathcal{U}(\gamma(r)),
\end{align*}
which corresponds to the set of solutions to \eqref{eq:sensitivityDI} with the initial condition given by $x = \gamma(r)$ in \eqref{eq:mainODE}. 

\begin{lemma}
			The map	$\mathcal{U}_\gamma$ is locally bounded, has a closed graph, is Borel measurable and admits a countable collection of dense Carath\'eodory selections $M \colon [0,1] \times [0,T] \mapsto \RR^{p \times p}$ which are absolutely continuous in time and Borel measurable in $r$. For each such $M$, there is a Lebesgue measurable selection $S(r,t) \in  J_F(\phi(\gamma(r),t)$ for all $(r,t) \in  [0,1] \times \RR$, such that, for all $r\in [0,1]$ and for almost all $t \in [0,T]$,
				\begin{align*}
								\frac{\partial }{\partial t}M(r,t) = S(r,t) M(r,t).
				\end{align*}
        \label{lem:caratheodoryCastaingGamma}
\end{lemma}
\begin{proof}
	Since $J_F$ is bounded, $\phi$ is Lipschitz and $\gamma$ is absolutely continuous, one can deduce that $\mathcal{U}_\gamma$ is locally bounded. Using the fact that $\gamma$ is absolutely continuous (and therefore has a closed graph) and invoking \cite[Corollary 1 and Theorem 3, \S 7, Chapter 2]{filippov1988differential}, one can deduce that $\mathcal{U}_\gamma$ has a closed graph.
		
		Using Proposition \ref{prop:castaingRepresentation}, $\mathcal{U}_\gamma$ admits a Castaing representation as the closure of a countable dense set of Carath\'eodory selections.
		It remains to show that the functions composing this representation satisfy the claimed differential equation and to construct the proposed $S$.
		
		Fix $M$ an element of this Castaing representation. With a slight abuse of notation, for the rest of this proof, we will see $M$ as a function of $(r,t)$ by identifying $M$ with $(r,t) \mapsto M(r,t)$. As a Carath\'eodory function, $M$ is jointly Borel measurable in $r$ and $t$ as stated in \cite[Lemma 4.51]{aliprantis2005infinite}. Denote by $\frac{\partial}{\partial t}M$ the partial derivative of $M$ with respect to $t$ when it exists. Since $M$ is Lipschitz (hence absolutely continuous) with respect to $t$, for any $r\in [0,1]$, $\frac{\partial }{\partial t}M(r,t)$ is defined for almost all $t\in[0,T]$. 
		
		Consider the set $E\subset [0,1]\times [0,T]$ the set where $\frac{\partial}{\partial t} M(r,t)$ exists. By Lemma \ref{lemma-direct}, $E$ has full Lebesgue measure and $(r,t) \mapsto \frac{\partial}{\partial t} M(r,t)$ is Lebesgue measurable. Furthermore, for all $r\in [0,1]$, $\lbrace t\in [0,T], (r,t)\in E\rbrace$ has full measure by Lipschicity of $M$ in the variable $t$ for fixed $r$. The set $E$ is a measure space with the induced subspace measure. 
		
		Consider the function 
		\begin{align*}
						f \colon \RR^{p \times p} \times E &\rightarrow \RR_+ \\
						(S,r,t) &\mapsto \left\Vert \frac{\partial}{\partial t} M(r,t) - SM(r,t)\right\Vert^2 ,
		\end{align*}
		which is jointly Lebesgue measurable in $(r,t)$  for a fixed $S\in\mathbb{R}^{p\times p}$ since the sum of the Lebesgue measurable functions $\frac{\partial}{\partial t}M$ and $-SM(r,t)$ is Lebesgue measurable, and because the composition of this sum with the norm function (which is continuous) is also Lebesgue measurable. This function is also continuous in $S$ for fixed $(r,t)\in [0,1] \times [0,T]$, implying then that $f$ is a Carath\'edory function. Consider $K > 0$ the global upper bound on $\|J_F\|_{op}$ as in \eqref{eq:boundJacobian}. By \cite[Corollary 18.8]{aliprantis2005infinite}, the set valued map
		\begin{align*}
						\mathcal{S}_1 \colon E &\rightrightarrows \RR^{p \times p} \\
						(r,t) &\rightrightarrows \left\{ S \in \RR^{p\times p},\, \|S\| \leq K,\, f(S,r,t) = 0 \right\}
		\end{align*}
		is measurable since $S$ belongs to a compact set.  
		We extend $\mathcal{S}_1$ to $[0,1] \times [0,T]$ by setting $\mathcal{S}_1(r,t) = \emptyset$ if $(r,t) \not \in E$. Measurability of $\mathcal{S}_1$ is preserved applying \cite[Definition 18.1]{aliprantis2005infinite}. Now consider the set valued function
		\begin{align*}
						\mathcal{S}_2 \colon [0,1] \times [0,T] &\rightrightarrows \RR^p \\
						(r,t) & \rightrightarrows J_F(\phi(\gamma(r),t)).
		\end{align*}
		Since the graph of $J_F$ is closed, and using moreover the continuity of the functions $\phi$ and $\gamma$, the function $S \mapsto \mathrm{dist}(S, J_F(\phi(\gamma(r),t)))$ is lower semicontinuous, hence Borel measurable by Lemma \ref{lem:lower-measurable}. It implies that it is a Carath\'eodory function,  proving that $\mathcal{S}_2$ is Borel measurable \cite[Theorem 18.5]{aliprantis2005infinite}. Now consider the intersection set valued map:
		\begin{align*}
						\mathcal{S} \colon [0,1] \times [0,T] &\rightrightarrows \RR^p \\
						(r,t) & \mapsto \mathcal{S}_1(r,t) \cap \mathcal{S}_2(r,t).
		\end{align*}
		It is measurable \cite[Lemma 18.4, Item 3]{aliprantis2005infinite} and compact valued.  Consider the set $\tilde{E} = \{(r,t) \in [0,1] \times [0,T],\, \mathcal{S}(x,t) \neq \emptyset \}$, which is measurable (see discussion after Definition 18.1 in \cite{aliprantis2005infinite}). We have that $\tilde{E} \subset E$ because $\mathcal{S}$ is empty valued outside of $E$ and $\tilde{E} = \{(r,t) \in E, \frac{d}{dt} M(r,t) \in J_F(\phi(\gamma(r),t)) M(r,t) \}$ by construction.
		
		Since for any $r\in [0,1]$, $M(r) \in \mathcal{U}_\gamma(r)$, it holds for almost all $t \in [0,T]$ that $\frac{d}{d t}M(r,t)\in J_F(\phi(\gamma(r),t)) M(r,t)$. In other words, for all $r \in [0,1]$, $\{t \in [0,T],\, (r,t) \in \tilde{E}\}$ has full measure. This is by definition of $\mathcal{U}_\gamma$. Therefore, by Fubini's Theorem \cite[Theorem 16 Section 20.2]{royden2010real}, $\tilde{E}$ has full measure.  
		
		Set for all $(r,t) \in [0,1] \times [0,T]$, $\tilde{\mathcal{S}}(r,t) = \mathcal{S}(r,t)$ if $\mathcal{S}(r,t) \neq \emptyset$ (that is $(r,t) \in \tilde{E}$), and $J_F(\phi(\gamma(r)),t)$ otherwise, it satisfies $\tilde{\mathcal{S}}(r,t) \subset J_F(\phi(\gamma(r),t))$ for all $(r,t) \in [0,1] \times [0,T]$ and has nonempty values. The mapping $\tilde{\mathcal{S}}$ is measurable and has non empty closed values \cite[Theorem 18.13]{aliprantis2005infinite}. Therefore it admits a measurable  selection $S \colon [0,1] \times [0,T] \mapsto \RR^{p \times p}$, which is the desired function. This achieves the proof. \end{proof}

\begin{remark}
				Given $M$ and $S$ as in Lemma \ref{lem:caratheodoryCastaingGamma}, we have by \cite[Theorem 2, \S 1, Chapter1]{filippov1988differential} that, for all $r\in [0,1]$, $t \mapsto M(r,t)$ is the unique absolutely continuous solution to
				\begin{align*}
								\frac{\partial }{\partial t} M(r,t) = S(r,t) M(r,t).
				\end{align*}
				\label{rem:unique}
\end{remark}

\begin{remark}
\label{rem:Uclosedgraph}
Note that, using the same arguments, the solution mapping $\mathcal{U}$ defined at the beginning of the section is also locally bounded and has a closed graph. Indeed, since $J_F$ is bounded and $\phi$ is Lipschitz, it is clear that $\mathcal{U}$ is locally bounded. Then using \cite[Corally 1 and Theorem 3, \S 7,Chapter 2]{filippov1988differential}, one deduces that $\mathcal{U}$ has a closed graph.
\end{remark}

\section{Path differentiability of the flow}
\label{sec_flow}
This section is devoted to the proof of our main result, conservativity of the mapping defined in \eqref{eq:sensitivityDI} for the flow of \eqref{eq:mainODE}. We first prove that, for any $T\geq 0$, the mapping $U$ evaluated at $t=T$ is conservative for the flow evaluated at $t=T$. 
\begin{theorem}
				For all $T\geq 0$, the mapping $x \rightrightarrows U(x,T)$ is conservative for $x \mapsto \phi(x,T)$.
				\label{th:flowPathDifferentiable}
\end{theorem}
\begin{proof}
         If $T = 0$, then the statement is obvious. Then, we restrict our analysis to the case where $ T > 0$.

				Consider an absolutely continuous path $\gamma \colon [0,1] \rightarrow  \RR^p$. Let $M\colon [0,1] \times [0,T] \rightarrow \RR^{p \times p}$ be a Carath\'eodory function as in Proposition \ref{prop:castaingRepresentation} such that, for all $(r,t)\in [0,1]\times [0,T]$, $M(r,t) \in U(\gamma(r),t)$. Consider the Lebesgue measurable selection $S(r,t) \in J_F(\phi(\gamma(r),t))$ for all $(r,t)\in [0,1]\times [0,T]$ as given by Lemma \ref{lem:caratheodoryCastaingGamma}, such that, for all $r\in [0,1]$ and almost all $t\in [0,T]$
				\begin{align}
								\frac{\partial}{\partial t} M(r,t) = S(r,t) M(r,t).
								\label{eq:selectionM}
				\end{align}
				In addition, we have, for all $r \in [0,1]$ and all $t \in [0,T]$,
				\begin{align}
								\phi(\gamma(r), t) - \gamma(r) = \int_{s=0}^{s=t} F(\phi(\gamma(r), s)) ds,
								\label{eq:integralEquation}
				\end{align}
				since $\phi(\gamma(r),0) = \gamma(r)$.
				
				Since $\phi$ is Lipschitz, for each $s \in [0,T]$, $r \mapsto \phi(\gamma(r),s)$ is an absolutely continuous loop. Therefore, it is differentiable at almost all $r\in [0,1]$. Applying Lemma \ref{lemma-direct} shows that the function
				\begin{align*}
								g \colon (r,s) \mapsto \frac{d}{dr} \phi(\gamma(r), s),
				\end{align*}
				is well defined for all $s\in [0,t]$ and almost all $r\in [0,1]$, and Lebesgue measurable in $(r,s)$. Therefore, for all $t\in [0,T]$, for almost all $r\in [0,1]$, it follows from \eqref{eq:integralEquation} that 
				\begin{align*}
							g(r,t) - \dot{\gamma}(r) = \frac{d}{dr} \int_{s=0}^{s=t} F(\phi(\gamma(r), s)) ds. 
				\end{align*}
				The integrand is jointly integrable in $(r,s)$, and absolutely continuous in $r$ for each $s$. It follows by Lemma \ref{lem:leibniz} that, for all $t\geq 0$ and for almost all $r\in [0,1]$
				\begin{align*}
						g(r,t) - \dot{\gamma}(r) =  \int_{s=0}^{s=t} \frac{\partial}{\partial r} F(\phi(\gamma(r), s)) ds.
				\end{align*}
				Since $F$ is path differentiable, we have, for all $s\in [0,t]$, for almost all $r\in [0,1]$
				\begin{align*}
					\frac{\partial }{\partial r} F(\phi(\gamma(r), s)) & = J \times g(r,s) \qquad \forall J \in J_F(\phi(\gamma(r), s))\\
								&=S(r,s)g(r,s),
				\end{align*}
				where $S$ is the Lebesgue measurable selection defined in \eqref{eq:selectionM}. 
	     	    Therefore, by integration, we have, for all $t \in [0,T]$, for almost all $r \in [0,1]$
				\begin{align}
								g(r,t) - \dot{\gamma}(r) =  \int_{s=0}^{s=t} S(r, s) g(r,s) ds,
								\label{eq:temp1}
				\end{align}

				Now, we rewrite \eqref{eq:selectionM} by integration, for all $(r,t)\in [0,1]\times [0,T]$, using $M(r,0) = I$
				\begin{align*}
					M(r,t)- I = \int_{s = 0}^{s = t} S(r, s) M(r,s) ds. 
				\end{align*}
				Multiplying both sides of the latter equation by $\dot{\gamma}(r)$, that is defined for almost all $r\in [0,1]$, one has, for all $t\geq 0$, for almost all $r\in [0,1]$ 
				\begin{align}
								M(r, t) \dot{\gamma}(r)	- \dot{\gamma}(r) = \int_{s=0}^{s=t} S(r, s) M(r, t) \dot{\gamma}(r) ds.
								\label{eq:temp2}
				\end{align}
				Combining both \eqref{eq:temp1} and \eqref{eq:temp2}, we have, for all $t\geq 0$, for almost all $r\in [0,1]$
				\begin{align*}
								\|M(r, t) \dot{\gamma}(r) -  g(r,t)\| &= \left\|\int_{s=0}^{s=t} S(r, s) (M(r, s) \dot{\gamma}(r) - g(r,s)) ds  \right\| \\
								& \leq \int_{s=0}^{s=t} \left\|S(r, s) (M(r, s) \dot{\gamma}(r) - g(r,s))   \right\| ds\\
								& \leq K\int_{s=0}^{s=t} \left\|(M(r, s) \dot{\gamma}(r) - g(r,s))   \right\| ds
				\end{align*}
				where $K$ is a bound on $J_F$ given in \eqref{eq:boundJacobian}. Integrating with respect to $r$ and using Fubini's theorem, we have, for all $t\in [0,T]$
				\begin{align*}
								\int_{r=0}^{r=1} \|M(r, t) \dot{\gamma}(r) -  g(r,t)\| dr & \leq K\int_{s=0}^{s=t}\int_{r=0}^{r=1}  \left\|(M(r, s) \dot{\gamma}(r) - g(r,s))   \right\| dr ds
				\end{align*}

				By Lemma \ref{lem:gronwall_int}, one obtains that, for all $t\geq 0$
				\begin{align*}
								\int_{r=0}^{r=1} \|M(r, t) \dot{\gamma}(r) -  g(r,t)\| dr = 0. 
				\end{align*}
				Therefore, we have, for all $t\geq 0$ and all $r \in [0,1]$ 
				\begin{align*}
								\phi(\gamma(r),t) - \phi(\gamma(0),t) = \int_{u=0}^{u=r} g(u,t) du =  \int_{u=0}^{u=r} M(u, t) \dot{\gamma}(u)du.
				\end{align*}
				Since $M$ was an arbitrary Carath\'eodory function in a countable dense subset of such selections, one can apply Lemma \ref{lem:dense}. This shows that $x \rightrightarrows U(x,t)$ is conservative for $x \mapsto \phi(x,t)$. This concludes the proof.\end{proof}
				
From the latter result, one can deduce that the flow $\phi$ is path differentiable for all $t\geq 0$. It is stated in the following corollary. 

\begin{corollary}
                The mapping $(x,t) \rightrightarrows (U(x,t), F( \phi(x,t)))$ is conservative for $\phi$ and in particular, $\phi$ is path differentiable.
				\label{cor:flowPathDifferentiable}
\end{corollary}
\begin{proof}
                Consider the following dynamical system on $\RR^{p+1}$
                \begin{align}
                    \label{eq:flowPrime}
                    \dot{Y}(s) &= \alpha(s) F(Y(s))\\
                    \dot{\alpha}(s) &= 0 \nonumber
                \end{align}
							 Consider $\tilde{F} \colon \RR^{p+1} \rightarrow \RR^{p+1}$ the vector field associated to the ODE in \eqref{eq:flowPrime} with the state $(Y,\alpha)$. It is given by
					$$
				\tilde{F}(Y(s),\alpha) = \begin{pmatrix}
					 \alpha F(Y(s))\\ 0
					\end{pmatrix}
					$$
					We can compute a conservative Jacobian for $\tilde{F}$ using the product rule of differential calculus and component-wise aggregation, both valid for conservative Jacobians  \cite[Lemmas 3 and 5]{bolte2020conservative}.
				    We obtain a conservative Jacobian for $\tilde{F}$ as follows:
					
								\begin{align}
												(x, \alpha) \rightrightarrows 
												\begin{pmatrix}
																\alpha J_F(x) & F(x)\\
																0&0
												\end{pmatrix}
												\label{eq:conservativeTildeF}
								\end{align}

								Denote by $\tilde{\phi} \colon \RR^{p+1} \rightarrow \RR^{p+1}$ the flow associated to \eqref{eq:flowPrime}, and recall that $\phi$ is the flow of the system \eqref{eq:mainODE}. We have, by a simple rescaling of time, for any $x \in \RR^p$, $\alpha \in \RR$ and any $s \in [0,1]$
								\begin{align}
												\phi(x,\alpha s) = \tilde{\phi}(x,\alpha,s).
												\label{eq:flowIdentity}
								\end{align}
								Setting $\alpha=t$ and $s=1$, by Theorem \ref{th:flowPathDifferentiable}, the mapping $(x,t) \mapsto \tilde{\phi}(x,t,1) = \phi(x,t)$ is path differentiable jointly in $(x,t)$. Let us compute a conservative Jacobian from Theorem \ref{th:flowPathDifferentiable}. The differential inclusion in \eqref{eq:sensitivityDI} can be expressed blockwise. Fix $x_0 \in \RR^p$ and $\alpha_0 \in \RR$ initial conditions for \eqref{eq:flowPrime} and denote by $Y \colon [0,1] \rightarrow \RR^p$ the solution to \eqref{eq:flowPrime}, note that $\alpha(t) = \alpha_0$ for all $t\in [0,1]$. Then, one has, for all $t \in [0,1]$, $Y(s) = X(\alpha_0 s)$ where $X$ is the solution to \eqref{eq:mainODE} starting at $x_0$. Moreover, one has
								\begin{align}
												\begin{pmatrix}
																\dot{V_1}(s)& \dot{V_2}(s) \\
																\dot{V_3}(s)& \dot{V_4}(s)
												\end{pmatrix}
												\in  
												\begin{pmatrix}
																\alpha_0 J_F(Y(s))V_1(s) + F(Y(s)) V_3(s)& \alpha_0 J_F(Y(s))V_2(s) + F(Y(s)) V_4(s) \\
																0 & 0
												\end{pmatrix},
												\label{eq:MflowPrime}
								\end{align}
								where $V_1 \in \RR^{p \times p}$ and $V_1(0)$ is the identity, $V_2 \in \RR^{p \times 1}$ and $V_2(0) = 0$, $V_3 \in \RR^{1 \times p}$ and $V_3(0)=0$, $V_4 \in \RR$ and $V_4(0) = 1$. It follows that $V_3 =0$ and $V_4 =1$ for all $t$ and
								\begin{align}
												\dot{V_1}(s) &\in \alpha_0 J_F(Y(s))V_1(s) = \alpha_0 J_F(X(\alpha_{0} s))V_1(s) \nonumber\\
												\dot{V_2}(s) &\in \alpha_0 J_F(Y(s))V_2(s) + F(Y(s)). 
												\label{eq:MflowPrime2}
								\end{align}
								The two dynamics are independant. Furthermore, solutions of the first line are also solutions of \eqref{eq:sensitivityDI} modulo a simple time rescaling by a factor $\alpha_0$. This is more explicitly written $V_1(s) \in U(x_0, \alpha s)$ for all $s \in [0,1]$, where $U$ is given in \eqref{eq:candidateConservativeField}. Conversely, any $V \in U(x_0, \alpha s)$ is related to a solution of the first line of \eqref{eq:MflowPrime2}.  Let us show that $V_2 \colon t \mapsto sF(X(\alpha_0 s))$ is the unique solution to the second line. By path differentiability of $F$, the function $s \mapsto F(X(\alpha_0 s))$ is differentiable for almost all $t$, such that
								\begin{align*}
												\frac{d}{ds} F(X(\alpha_0 s)) &= J(X(\alpha_0 s)) \frac{d}{ds} X(\alpha_0 s)) & \forall J \in J_F(X(\alpha_0 s))\\
												& = \alpha_0 J(X(\alpha_0 s))F(X(\alpha_0 s)) & \forall J \in J_F(X(\alpha_0 s))
								\end{align*}
								The function $s \mapsto s F(X(\alpha_0 s))$ is absolutely continuous and multiplication by $s$ is a differentiable operation. Then, for almost all $s \in [0,1]$, substituting $Y$ for $X$
								\begin{align}
												\frac{d}{ds} [sF(Y(s))] & = \alpha_0 J(Y(s)) \left[ sF(Y(s))\right] + F(Y(s)) & \forall J \in J_F(Y(s))
												\label{eq:explicitSolutionM2}
								\end{align}
								Now, given a measurable selection in $s \mapsto S(s) \in J_F(Y(s))$, the function $(s, V_2) \mapsto S(s) V_2$ is Lipschitz in its second argument, so that the corresponding solution $V_2$ in \eqref{eq:MflowPrime2} is unique \cite[Theorem 2, \S 1, Chapter 1]{filippov1988differential}. Moreover, by \eqref{eq:explicitSolutionM2}, since $0F(Y(0)) = 0$, we have $V_2(s) = sF(Y(s))$ for all $s \in [0,1]$. This shows that any solution to \eqref{eq:MflowPrime} is given by $V_3 = 0$, $V_4 = 1$, and for all $s \in [0,1]$,
								\begin{align*}
												V_1(s) &\in U(x_0, \alpha_0 s)\\
												V_2(s) &= sF(X(\alpha_0 s)).
								\end{align*}
								Thanks to Theorem \ref{th:flowPathDifferentiable}, we have
								\begin{align*}
												(x,\alpha) \rightrightarrows (U(x, \alpha), F(\phi(x,\alpha))),
								\end{align*}
								is conservative for the mapping $(x,\alpha) \mapsto \tilde{\phi}(x,\alpha,1)$.
								Using the fact that $\phi(x,\alpha) = \tilde{\phi}(x,\alpha,1)$ for all $x \in \RR^p$, $\alpha \in \RR$, this proves the desired result using $\alpha=t$.
\end{proof}

\section{Consequences: backward and forward derivatives}

\label{sec_adjoint}

In this section, we focus on an optimization of integral costs under ODE constraint and prove that, as soon as the ODE vector field and the integrand are path differentiable, then the integral cost is itself path differentiable. One should see these results as consecutive results of Sections \ref{sec_preliminary} and \ref{sec_flow}. We provide further results about forward and backward derivatives propagation with a nonsmooth adjoint system.

\subsection{Differentiation of a terminal cost}
The following result is a direct consequence of Theorem \ref{th:flowPathDifferentiable} and the fact that product of conservative Jacobian is a also conservative Jacobian, as stated in \cite[Lemma 5]{bolte2020conservative}.
\begin{corollary}
				Let $\delta_T \colon \RR^p \rightarrow \RR$ be locally Lipschitz and path differentiable. Let $D_{\delta_T} \colon \RR^p \rightrightarrows \RR^p$ be conservative gradient for $\delta_T$.
				Then the following set
				\begin{align}
								D_T \colon x \rightrightarrows  \left\{ V^\top u ,\, V \in U(x,T),\, u \in D_{\delta_T}(\phi(x,T))  \right\}
								\label{eq:conservativeIntegralLoss1}
				\end{align}
				is a conservative gradient for $x \mapsto \delta_T(\phi(x,T))$.
				\label{lem:conservativeIntegralLoss1}
\end{corollary}

\subsection{Forward propagation of derivatives of integral costs}

In this subsection, we show how our framework allows to compute forward derivatives of integral costs. Such results already exist in a nonsmooth context with other classes of functions such as the functions admitting lexicographic derivatives (see e.g., \cite{barton2018computationally}). Note that this result will be instrumental in deriving a backward derivative propagation in the form of an adjoint system. 

\begin{theorem}
				Let $\delta \colon \RR^p \rightarrow \RR$ be locally Lipschitz and path differentiable. Let $D_\delta \colon \RR^p \rightrightarrows \RR^p$ be a conservative Jacobian for $\delta$, with convex values. For $T >  0$, set
				\begin{align}
								\Delta(x) =  \int_{t=0}^{t=T} \delta(\phi(x,t))dt.
								\label{eq:integralLoss1}
				\end{align}
				Then the following set valued field is a conservative gradient for $\Delta$,
				\begin{align}
								D_\Delta \colon x \rightrightarrows  \left\{ \int_{t=0}^{t=T} V(t)^\top w(t) dt,\, V \in \mathcal{U}(x),\ w \in \mathcal{W}(x)  \right\}
								\label{eq:conservativeIntegralLoss2}
				\end{align}
				where $\mathcal{W}(x)$ is the set of measurable selections $w(t) \in D_\delta(\phi(x,t))$ for all $t \in [0,T]$ and $x \in \RR^p$. In particular $V$ could be any solution of $\eqref{eq:sensitivityDI}$.
				\label{th:conservativeIntegralLoss2}
\end{theorem}
\begin{proof}
				We prove first that $D_\Delta$ has a closed graph, nonempty values and is locally bounded. Both $\mathcal{U}(x)$ and $\mathcal{W}(x)$ are nonempty valued and locally bounded thanks to the local boundedness of $D_\delta$, $\phi$ and $\mathcal{U}$ proved in Lemma \ref{lem:lipschitzSolution}.
				Therefore $D_\Delta$ is locally bounded and nonempty valued. Second, we sketch the proof of graph closedness. Consider a sequence $(x_k)_{k \in \NN}$ converging to $\bar{x}$, and $(d_k)_{k \in \NN}$ converging to $\bar{d}$, such that, for each $k \in \NN$, the sequence $(d_k)_{k\in\mathbb{N}}$ is defined by
				\begin{align*}
								d_k = \left\{  \int_{t=0}^{t=T} V_k(t)^\top w_k(t) dt,\, V_k \in \mathcal{U}(x_k),\, w_k \in \mathcal{W}(x_k) \right\}.
				\end{align*}
				The sequence $(V_k)_{k\in \NN}$ is bounded and Lipschitz (as proven in Lemma \ref{lem:lipschitzSolution}) uniformly in $k$, therefore we can use the Arzel\'a-Ascoli's Theorem \cite[Theorem 4.25]{brezis2010functional}: up to a subsequence, $V_k$ converges to a given $\bar{V}$ uniformly on $[0,T]$. As detailed in Remark \ref{rem:Uclosedgraph}, it holds that $\bar{V} \in \mathcal{U}(\bar{x})$. The sequence $(w_k)_{k \in \NN}$ is bounded in $L^2([0,T])$ (and in $L^\infty([0,T])$) so it has a weakly convergent subsequence by \cite[Theorem 17, Section 14]{royden2010real} whose limit will be denoted by $\bar{w}\colon  [0,T] \rightarrow \RR^p$. Up to a convex combination, the convergence occurs strongly and therefore pointwise almost everywhere by invoking Mazur's Lemma \cite[Corollary 3.8]{brezis2010functional}. We deduce that $\bar{w}(t) \in D_\delta(\phi(\bar{x},t))$ for almost all $t \in [0,T]$ using the fact that $D_\delta$ has convex values. Combining uniform convergence of $V_k$ to $\bar{V} \in \mathcal{U}(\bar{x})$ and weak convergence of $w_k$ to $\bar{w}$, we have that $d_k \to \int_{t=0}^{t=T} \bar{V}(t)^\top \bar{w}(t) dt \in D_\Delta(\bar{x})$ and $\bar{d} \in D_\Delta(\bar{x})$ by uniqueness of the limit. 

                From now on, we fix a Borel measurable selection $d_\Delta$ such that $d_\Delta(x) \in D_\Delta(x)$ for all $x \in \RR^p$. This means that, for all $x\in\RR^p$, there is a continuous function $V_x \in \mathcal{U}(x)$ and a measurable function $w_x \in \mathcal{W}(x)$ such that 
                \begin{align*}
                    d_\Delta(x) = \int_{t=0}^{t=T} V_x(t)^\top w_x(t) dt.
                \end{align*}

				Now, fix an absolutely continuous path $\gamma \colon [0,1] \rightarrow \RR^p$. Since $\delta$ and $\phi$ are Lipschitz functions, and since $\gamma$ is absolutely continuous, we have that
				\begin{align*}
								r \mapsto \Delta(\gamma(r)) := \int_{t=0}^{t=T} \delta(\phi(\gamma(r),t))dt,
				\end{align*}
				is absolutely continuous. 
				By Corollary \ref{lem:conservativeIntegralLoss1}, for all $t\in [0,T]$, for a.e. $r \in [0,1]$,
				\begin{align}
								\frac{\partial }{\partial r}\delta(\phi(\gamma(r),t)) &= \dot{\gamma}(r) ^\top M^\top  v,  &\forall v \in D_\delta(\phi(\gamma(r),t)),\, \forall M \in U(\phi(\gamma(r),t)). 
					    \label{eq:almostAllMv}
				\end{align}
				Denote by $E \subset [0,1] \times [0,T]$ the set where \eqref{eq:almostAllMv} holds. Let us show that this set is Lebesgue measurable.
				
				Consider the function
                \begin{align*}
                    f \colon (M,v,r,t) \mapsto \left\|\frac{\partial }{\partial r}\delta(\phi(\gamma(r),t)) - \dot{\gamma}(r)^TM^Tv \right\|^2
                \end{align*}
                if $\frac{\partial }{\partial r}\delta(\phi(\gamma(r),t)) $ and $\dot{\gamma}(r)$ are well defined, and $1$ otherwise.
                The function $f$ is jointly Lebesgue measurable  in $(r,t)$ for fixed $M$ and $v$ and jointly continuous in $(M,v)$ for fixed $(r,t)\in [0,1]\times \mathbb{R}_+$. Then, it is a Carath\'eodory function. Therefore, the function
                \begin{align*}
                    \tilde{f} \colon (r,t) \mapsto \max\quad & \left\|\frac{\partial }{\partial r}\delta(\phi(\gamma(r),t)) - \dot{\gamma}(r)^TM^Tv \right\|^2 \\
                    \mathrm{s.t.} \quad & v \in D_\delta(\phi(\gamma(t),t)) \\
                    &M \in U(\phi(\gamma(r),t)) 
                \end{align*}
                is Lebesgue measurable thanks to \cite[Theorem 18.19]{aliprantis2005infinite}.
                The set $\{(r,t)\in [0,1]\times \mathbb{R}_+, \tilde{f}(r,t)~= 0\}$ is Lebesgue measurable and corresponds to the set where \eqref{eq:almostAllMv} holds. Therefore, \eqref{eq:almostAllMv} holds on a jointly measurable set. Furthermore, since \eqref{eq:almostAllMv} holds for all $t\in [0,T]$ for almost all $r\in [0,1]$, $E$ has actually full measure.

				Now consider the set
				\begin{align*}
				    S = \left\{ (r,t) \in [0,1] \times [0,T],\,  \frac{\partial }{\partial r}\delta(\phi(\gamma(r),t)) = \dot{\gamma}(r) ^\top V_{\gamma(r)}(t) ^\top  w_{\gamma(r)}(t)\right\}.
				\end{align*}
				Clearly, $E \subset S$ so that $S^c \subset E^c$. Moreover, since $E^c$ has zero measure we deduce that $S^c$ has zero (Lebesgue) measure. Therefore $S$ is measurable jointly in $(r,t)$ and the function $(r,t) \to \dot{\gamma}(r) V_{\gamma(r)}(t) w_{\gamma(r)}(t)$ is also (Lebesgue) measurable \cite[Proposition 3, Section 18.1]{royden2010real}. From Lemma \ref{lem:leibniz}, we have that $r \mapsto \Delta(\gamma(r))$ is absolutely continuous and for almost all $r\in [0,1]$, 
				\begin{align*}
				    \frac{d}{dr} \Delta(\gamma(r)) &= \int_{t=0}^{t=T} \frac{\partial }{\partial r}\delta(\phi(\gamma(r),t)) dt \\
				    &= \int_{t=0}^{t=T} \dot{\gamma}(r) ^\top V_{\gamma(r)}(t) ^\top  w_{\gamma(r)}(t)dt \\
				    &= \dot{\gamma}(r) ^\top \int_{t=0}^{t=T}  V_{\gamma(r)}(t) ^\top  w_{\gamma(r)}(t)dt \\
				    &= \dot{\gamma}(r) ^\top d_\Delta(r).
				\end{align*}
				Note that $d_\Delta$ was an arbitrary measurable selection in $D_\Delta$. Since $D_\Delta$ has a closed graph, it admits a countable Castaing representation (\cite[Corollary 18.14]{aliprantis2005infinite} and \cite[Theorem 18.20]{aliprantis2005infinite}). Then Lemma \ref{lem:dense} applies and conservativity is proved, which leads to the desired result.
\end{proof}

\subsection{Path differentiable adjoint method for integral costs}

We describe a path differentiable version of the adjoint method for integral costs under ODE constraints. 
\begin{corollary}
				Let $\delta,\delta_T \colon \RR^p \mapsto \RR$ be locally Lipschitz and path differentiable functions. Let $D_\delta \colon \RR^p \rightrightarrows \RR^p$ and $D_{\delta_T}\colon\mathbb{R}^p \rightrightarrows \RR^p$ be conservative Jacobians for $\delta$ and $\delta_T$, respectively where $D_\delta$ has convex values.  
				
				For any $x \in \RR^p$, any $w \colon [0,T] \to \RR^p$ measurable such that $w(t) \in D_\delta(\phi(x,t))$ for all $t \in [0,T]$, any $J \colon [0,T] \to \RR^{p\times p}$ measurable such that $J(t) \in J_F(\phi(x,t))$ for all $t \in [0,T]$ and any $u \in D_{\delta_T}(\phi(x,T))$, the unique absolutely continuous solution $\lambda \colon [0,T] \to \RR^p$ to the system
				\begin{align}
					\dot{\lambda}(t) &= -w(t) - J(t)^\top \lambda(t), \nonumber\\
					\lambda(T) &= u 
					\label{eq:conservativeIntegralLossBackward}
				\end{align}
			  satisfies $\lambda(0) \in D_\Delta(x) + D_{T}(x)$ where $D_\Delta$ and $D_T$ are defined in Corollary \ref{lem:conservativeIntegralLoss1} and Theorem \ref{th:conservativeIntegralLoss2}.
				\label{cor:conservativeIntegralLossBackward}
\end{corollary}

\begin{proof} Fix $x\in\mathbb{R}^p$. Fix $w$ and $J$ as in the statement of the corollary. This defines a unique $M \in \mathcal{U}(x)$ by solving \eqref{eq:sensitivityDI} $\dot{M}(t) = J(t)M(t)$ with $M(0) = I$ \cite[Theorem 2, \S 1, Chapter 1]{filippov1988differential}.

For any absolutely continuous function $\lambda \colon [0,T] \rightarrow \RR^p$, we have
\begin{align*}
				\int_{t=0}^{t=T} &M(t)^\top  w(t) dt =\\
				& \int_{t=0}^{t=T}M(t)^\top  w(t) +\left( J(t) M(t) - J(t) M(t) \right)^\top\lambda(t)dt
\end{align*}
Using Lemma \ref{lem:ipp}, we have
\begin{align*}
				 \int_{t=0}^{t=T}  (J(t) M(t))^\top\lambda(t) dt  = \int_{t=0}^{t=T}  \dot{M}(t)^\top \lambda(t) dt = &-\lambda(0) + M(T)^\top \lambda(T)\\
				&-\int_{0}^\top M(t)^\top \dot{\lambda}(t)  dt.
\end{align*}
Hence, we have for any $u \in D_T(\phi(x,T))$,
\begin{align*}
&M(T)^\top u +  \int_{t=0}^{t=T} M(t)^\top w(t) dt \\
=\;&\int_{t=0}^{t=T} M(t)^\top \left(w(t) + J(t)^\top \lambda(t) + \dot{\lambda}(t)\right)  dt + M(0)\lambda(0) + M(T)^\top (u- \lambda(T))
\end{align*}
The latter holds for any absolutely continuous function $\lambda$, and in particular, using \cite[Theorem 2, \S 1, Chapter 1]{filippov1988differential}, one can choose $\lambda$ as the unique absolutely continuous solution to 
\begin{align}
				\dot{\lambda}(t) &= - w(t) - J(t)^\top \lambda(t)\\
				\lambda(T) &=  u.
				\label{eq:backwardODE}
\end{align}
Using the fact that $M(0)$ is the identity, one has finally
\begin{align*}
				M(T)^\top u + \int_{t=0}^{t=T} M(t)^\top w(t) dt =& M(0)\lambda(0) = \lambda(0).
\end{align*}
The term $\lambda(0)$ being defined as the sum of two specific elements in $D_\Delta$ and $D_T$ (see Corollary \ref{lem:conservativeIntegralLoss1} and Lemma \ref{th:conservativeIntegralLoss2}), this means that $\lambda(0)\in D_\Delta(x) + D_T(x)$, which concludes the proof.
\end{proof}
\begin{remark}
    The system \eqref{eq:conservativeIntegralLossBackward} is typically solved backward in time. Setting $g \colon s \mapsto \lambda(T - s)$, we have, for all $s \in [0,T]$,
    \begin{align*}
        g(0) &= u\\
        \dot{g}(s) &= - \dot{\lambda}(T-s) = w(T-s) + J(T-s)\lambda(T-s),
    \end{align*}
    which is the backpropagation equation.
\end{remark}

\section{Minimization of integral costs with parameterized ODEs constraints}
\label{sec_neural}
This section is centered around problem \eqref{eq:integralintro}. The results combine conservative calculus rules with the elements developed in Section \ref{sec_adjoint}. 

\subsection{Problem setting}
We consider the optimization problem described in \eqref{eq:integralintro} and introduce further notations. First the constraints in \eqref{eq:integralintro} relate to the following parametrized ODE, given $T>0$, for all $t\in [0,T]$
\begin{equation}
\label{eq:param-ODE}
\dot{Z}(t):=H(Z(t),\theta), Z(0) = z,
\end{equation}
where $\theta\in\mathbb{R}^m$ denotes a vector of parameters and $H \colon \RR^{p + m} \to \RR^p$ is a Lipschitz path differentiable function. We assume that $J_H \colon \RR^{p + m} \rightrightarrows \RR^p$ is a conservative Jacobian for $H$ and that it is bounded. We denote $\psi(z,\theta,t) \in \RR^p$ the flow associated to the ODE \eqref{eq:param-ODE}. Throughout this section, $z$ will be a fixed initial condition in $\RR^p$.

We denote by $L_I \colon \RR^m \to \RR$ the integral part of the loss in \eqref{eq:integralintro}, 
\begin{align*}
    L_I \colon \theta \mapsto \int_0^T \ell(\psi(z,\theta,t)) dt.
\end{align*}
Recall that $\ell\colon \RR^p \to \RR$ is a locally Lipschitz, path differentiable functions from $\RR^p$ to $\RR$, we further assume that it admits a conservative gradients, $D_\ell \colon \RR^p \rightrightarrows \RR^p$ with convex values. 

Furthermore, we denote by $L_T \colon \RR^m \to \RR$ the integral part of the loss in \eqref{eq:integralintro}, 
\begin{align*}
    L_T \colon \theta \mapsto \ell_T(\psi(z,\theta,T)).
\end{align*}
Recall again that $\ell_T\colon \RR^p \to \RR$ a locally Lipschitz, path differentiable function and assume that it admits a conservative gradients $D_{\ell_T} \colon \RR^p \rightrightarrows \RR^p$.

With a slight abuse of notations, problem \eqref{eq:integralintro} can be reformulated equivalently as an unconstrained minimization problem with cost $L:=L_I + L_T$ with respect to variable $\theta$, that is, 
\begin{align}
\inf_{\theta\in\mathbb{R}^m} L(\theta) = \inf_{\theta\in\mathbb{R}^m} \int_0^T \ell(\psi(z,\theta,t)) dt + \ell_T(\psi(z,\theta,T)).\label{eq:param-opti}
\end{align}

The purpose of this section is to specify the results presented in Section \ref{sec_adjoint} to the parametrized ODE \eqref{eq:param-ODE} and cost \eqref{eq:param-opti} in order to address problem \eqref{eq:integralintro}. This will result in expressions for conservative gradients, $D_I$ for $L_I$ and $D_T$ for $L_T$. Setting $D_L = D_I + D_T$, the sum of these conservative gradients, we obtain a conservative gradient for $L$ \cite[Corollary 4]{bolte2020conservative}. To this end, we see the parametrized flow of \eqref{eq:param-ODE} as an unparametrized flow in a lifted space and justify formal differentiation operations using conservative calculus rules \cite{bolte2020conservative} to make connections with results presented in Section \ref{sec_adjoint}. This results in an adjoint method which can in turn be used as an oracle for $D_L$. This allows to provide a convergence result for the corresponding small step first order optimization method to seek solutions of problem \eqref{eq:param-opti}.

\subsection{Conservative Jacobian of the flow}
The following corollary reformulates Theorem \ref{th:flowPathDifferentiable} in the context of system \eqref{eq:param-ODE}.
\begin{corollary}
    \label{cor:conservativeJacFlowParam}
    Let $\psi$ be defined as in equation \eqref{eq:param-ODE} with $H \colon \RR^{p + m} \to \RR^p$ a Lipschitz path-differentiable function and $J_H \colon \RR^{p + m} \rightrightarrows \RR^p$ a bounded conservative Jacobian. Consider the matrix differential inclusion with unknown $M \in \RR^{p \times m}$, for almost all $t \in [0,T]$
    \begin{align}
    \label{eq:DIparam}
    \dot{M}(t) &= J_z(t) M(t) + J_\theta(t)\\
    \begin{pmatrix}
        J_z(t) & J_\theta(t)
    \end{pmatrix} &\in J_H(\psi(z,\theta,t),\theta), \nonumber
    \end{align}
    with initial condition $M(0) = 0_{pm} \in \RR^{p \times m}$, where we used block matrix notations in the second line. Then the set $\{M(T),\, \text{ $M$ solution to \eqref{eq:DIparam}}\}$ forms a conservative Jacobian for $\theta \mapsto \psi(z,\theta,T)$.
    
    In particular, for any $\ell_T \colon \RR^p \to \RR$, path differentiable with conservative gradient $D_{\ell_T}$ and for any $z \in \RR^p$, the following set valued map
    \begin{align*}
        D_T \colon (z, \theta) \rightrightarrows \left\{M(T)^T u,\, \text{ $M$ solution to \eqref{eq:DIparam}},\, u \in D_{\ell_T}(\psi(z,\theta,T)) \right\}
    \end{align*}
    is a conservative gradient for the function $L_T \colon \theta \mapsto \ell_T(\psi(z,\theta,T))$.
\end{corollary}
\begin{proof}
    We first introduce notations allowing to interpret the system \eqref{eq:param-ODE} as a system of the form \eqref{eq:mainODE} on an extended state space $\RR^{p + m}$.
    We rewrite \eqref{eq:param-ODE} as follows, for all $t\in [0,T]$:
    \begin{align}
    \label{eq:param-ODE2}
        &\dot{Z}(t)=H(Z(t),\theta(t)), \: \dot{\theta}(t)= 0,\\
        &Z(0) = z,\: \theta(0) = \theta,
    \end{align}
    We consider $X = \begin{pmatrix}
                    Z \\ \theta
    \end{pmatrix}\in\mathbb{R}^{p+m}$, the concatenation of the two variables $Z \in \RR^p$ and $\theta \in \RR^m$. We set, for all such $X$, 
    \begin{align*}
        F(X)=\begin{pmatrix}
                        H(Z,\theta)\\ 0
        \end{pmatrix} \in \RR^{p+m}.
    \end{align*}
    With these notations, \eqref{eq:param-ODE2} is equivalently rewritten as follows, for all $t\in [0,T]$
    \begin{align*}
        \dot{X}(t):=F(X(t)),
    \end{align*}
    In this case, \eqref{eq:param-ODE2} is in the same form than the one given in \eqref{eq:mainODE}, and the parameter $\theta$ is now seen as an initial condition.
    Setting
    \begin{align*}
        J_F(z,\theta)\colon (z,\theta)\rightrightarrows \begin{pmatrix}
                    J_H(z,\theta)\\
                    0 
        \end{pmatrix},
    \end{align*}
    we have that $J_F$ is a conservative Jacobian for $F$ since $J_H$ is a conservative Jacobian for $H$. The variational inclusion \eqref{eq:sensitivityDI} for \eqref{eq:param-ODE} can be written as follows, for a.e. $t\in [0,T]$
    \begin{align}
        \dot{M} = 
        \begin{pmatrix}
            \dot{M}_1& \dot{M}_2 \\
            \dot{M}_3& \dot{M}_4
        \end{pmatrix}
        \in
        \begin{pmatrix}
            J_H(\psi(z,\theta,t),\theta) M\\
            0
        \end{pmatrix}
        \label{eq:paramSensitivityDI}
    \end{align}
    where $M_1(0) = I_{p} \in \RR^{p \times p}$, $M_2(0) = 0_{pm} \in \RR^{p \times m}$, $M_3(0) =  0_{mp} \in \RR^{m \times p}$ and $M_4(0) = I_m \in \RR^{m \times m}$. From this equation, $M_3$ and $M_4$ remain constant. From \cite[Lemma 4]{bolte2020conservative} and Theorem \ref{th:flowPathDifferentiable}, the concatenation $\begin{pmatrix} M_1 & M_2\end{pmatrix}$ for all solutions to \eqref{eq:paramSensitivityDI} forms a conservative Jacobian for $(z,\theta) \mapsto \psi(z,\theta,T)$. Let us express this equation in a simpler form.
    
    For any $t \in [0,T]$, and any $J \in J_H(\psi(z,\theta,t),\theta)$, writing $J = \begin{pmatrix} J_z&J_\theta \end{pmatrix}$ where $J_z \in \RR^{p \times p}$ and $J_\theta \in \RR^{p \times m}$, with $M_3 =  0_{mp} \in \RR^{m \times p}$ and $M_4 = I_m \in \RR^{m \times m}$, we have
    \begin{align*}
         J M = 
         \begin{pmatrix} 
            J_z&J_\theta 
        \end{pmatrix}
         \begin{pmatrix}
            M_1& M_2 \\
            M_3& M_4
        \end{pmatrix} 
        =
        \begin{pmatrix} 
            J_z M_1 & J_z M_2 + J_\theta 
        \end{pmatrix}.
    \end{align*}
    Equation \eqref{eq:paramSensitivityDI} is equivalently rewritten, for a.e. $t\in [0,T]$
    \begin{align}
        \label{eq:DIparam0}
        \dot{M}_1(t) &= J_z(t) M_1(t)\\
        \dot{M}_2(t) &= J_z(t) M_2(t) + J_\theta(t)\nonumber\\
        \begin{pmatrix}
            J_z(t) & J_\theta(t)
        \end{pmatrix} &\in J_H(\psi(z,\theta,t),\theta), \nonumber
    \end{align}
    with $M_1(0) = I_{p} \in \RR^{p \times p}$ and $M_2(0) = 0_{pm} \in \RR^{p \times m}$.
    
    Using Theorem \ref{th:flowPathDifferentiable}, we have proved that concatenations of the form $\begin{pmatrix} M_1(T) & M_2(T) \end{pmatrix}$ where $M_1$ and $M_2$ are solutions of \eqref{eq:DIparam0} form a conservative Jacobian for $(z,\theta) \mapsto \psi(z,\theta,T)$. Focusing on the dependency in $\theta$ for fixed $z$, and invoking Lemma \ref{lem:projection}, component $M_2$ form a conservative field for $\theta \mapsto \psi(z,\theta,T)$. This proves the corollary.
\end{proof}

\begin{remark}
Using the latter argument, one can also include a dependency of the initial condition in $\theta$, that is $Z(0) = z_0(\theta)$. It suffices to notice that this is a composition of the function $(z,\theta) \to \psi(z,\theta,T)$ for which we have a conservative Jacobian from \eqref{eq:DIparam0} and the function $\theta \to (z_0(\theta), \theta)$ for which we have a conservative Jacobian as long as we know a conservative Jacobian for $\theta \to z_0(\theta)$. We may then apply the composition rule \cite[Lemma 5]{bolte2020conservative}. It is also possible to include further dependency in $\theta$ for $\ell$ and $\ell_T$ with similar lifting techniques.
\end{remark}

\color{black}

\subsection{Differentiation of integral costs and adjoint method}
The following corollary is a reformulation of Theorem \ref{th:conservativeIntegralLoss2} in the context of \eqref{eq:param-ODE}, based on Corollary \ref{cor:conservativeJacFlowParam}.
\begin{corollary}
    Let $\psi$ be defined as in equation \eqref{eq:param-ODE} with $H \colon \RR^{p + m} \to \RR^p$ a Lipschitz path-differentiable function and $J_H \colon \RR^{p + m} \rightrightarrows \RR^p$ a bounded conservative Jacobian. Let  $\ell\colon \RR^p \to \RR$ be a locally Lipschitz, path differentiable functions with conservative gradients $D_\ell$ with convex values. For any $z\in\mathbb{R}^p$, consider the function $L_I:\theta\in\mathbb{R}^m \rightarrow \int_0^T \ell(\psi(z,\theta,t) dt$ and the set valued field:
    \begin{align*}
        D_I\colon (z,\theta)\rightrightarrows \left\{\int_{t=0}^{t=T} M(t)^\top w(t) dt,\: M \in \mathcal{U}(z,\theta),\, w \in \mathcal{W}(z,\theta)\right\}
    \end{align*}
    where $\mathcal{U}(z,\theta)$ is the set of solutions to \eqref{eq:DIparam} and $\mathcal{W}(z,\theta)$ is the set of measurable functions $w \colon [0,T] \to \RR^p$ such that $w(t) \in D_\ell(\psi(z, \theta,t))$ for all $t \in [0,T]$. For any $z \in \RR^p$, $\theta \rightrightarrows D_I(z,\theta)$ is a conservative gradient for $L$.
    \label{cor:conservativeJacIntParam}
\end{corollary}

\begin{proof}
Using Theorem \ref{th:conservativeIntegralLoss2} and the expressions given in \eqref{eq:paramSensitivityDI}, one knows that the set valued map:
\begin{equation*}
    (z,\theta) \rightrightarrows \left\{\int_{t=0}^{t=T} \begin{pmatrix}
        M_1^\top(t) & 0\\
        M_2^\top(t) & I_p
    \end{pmatrix} \begin{pmatrix}
        w(t)\\ 0
    \end{pmatrix} dt, \text{$M_1$ and $M_2$ solutions to \eqref{eq:DIparam0}}, w\in D_\ell(\psi(z,\theta,t))\right\}
\end{equation*}
is a conservative field for the function $(z,\theta) \mapsto \int_0^T \ell(\psi(z,\theta,t))dt $. Using Lemma \ref{lem:projection}, one can deduce that, for every $z\in\mathbb{R}^p$, the set valued field $\theta\rightrightarrows D_I(z,\theta)$ is a conservative gradient for $L_I$, which concludes the proof.  
\end{proof}

We have the following adaptation of the adjoint of Corollary \ref{cor:conservativeIntegralLossBackward}.

\begin{corollary}
    \label{cor:ajointParam}
    Let $\psi$ be defined as in equation \eqref{eq:param-ODE} with $H \colon \RR^{p + m} \to \RR^p$ a Lipschitz path-differentiable function and $J_H \colon \RR^{p + m} \rightrightarrows \RR^p$ a bounded conservative Jacobian. 
    Let $\ell\colon \RR^p \to \RR$ and $\ell_T \colon \RR^p \to \RR$ be locally Lipschitz, path differentiable functions from $\RR^p$ to $\RR$ with respective conservative gradients $D_\ell$ and $D_T$. 
    
    For any $z \in \RR^p$, $\theta \in \RR^m$, any $w \colon [0,T] \to \RR^p$ measurable such that $w(t) \in D_\ell(\psi(z,\theta,t))$ for all $t \in [0,T]$, any $J_z \colon [0,T] \to \RR^{p \times p}$ and $J_\theta \colon [0,T] \to \RR^{p \times m}$, measurable such that $\begin{pmatrix} J_z(t) & J_\theta(t) \end{pmatrix}\in J_H(\psi(z,\theta,t), \theta)$ for all $t \in [0,T]$ and any $u \in D_T(\psi(z,\theta,T))$, the unique absolutely continuous solution $\lambda \colon [0,T] \to \RR^p$ to the system
	\begin{align}
		\dot{\lambda}(t) &= -w(t) - J_z(t)^\top \lambda(t), \nonumber\\
		\lambda(T) &= u 
		\label{eq:conservativeIntegralLossBackwardParam}
	\end{align}
    satisfies $\int_0^T J_\theta(t)^\top \lambda(t) dt \in D_I(z,\theta) + D_T(z,\theta)$ which is an element of a conservative gradient for the loss function $L$ in \eqref{eq:param-opti}.
\end{corollary}

\begin{proof}
Fix $z\in\mathbb{R}^p$ and $\theta\in \mathbb{R}^m$. Fix $w$ and $J_z$ as in the statement of the theorem. This defines a unique $M \in \mathcal{U}(z,\theta)$ by solving \eqref{eq:conservativeIntegralLossBackwardParam} $\dot{M}(t) = J_z(t)M(t)+J_\theta(t)$ \cite[Theorem 2, \S 1, Chapter 1]{filippov1988differential}.

For any absolutely continuous function $\lambda \colon [0,T] \rightarrow \RR^p$, we have
\begin{align*}
				\int_{t=0}^{t=T} &M(t)^\top  w(t) dt =\\
				& \int_{t=0}^{t=T}M(t)^\top  w(t) +\left( J_z(t) M(t) + J_\theta(t) - J_z(t) M(t) - J_\theta(t) \right)^\top\lambda(t)dt
\end{align*}
Using Lemma \ref{lem:ipp}, we have
\begin{align*}
				 \int_{t=0}^{t=T}  (J_z(t) M(t)+J_\theta(t))^\top\lambda(t) dt  = \int_{t=0}^{t=T}  \dot{M}(t)^\top \lambda(t) dt = &-M(0)^\top\lambda(0) + M(T)^\top \lambda(T)\\
				&-\int_{0}^\top M(t)^\top \dot{\lambda}(t)  dt.
\end{align*}
Hence, since $M(0)=0$, we have, for any $u \in D_T(\psi(z,\theta,T))$,
\begin{align*}
&M(T)^\top u +  \int_{t=0}^{t=T} M(t)^\top w(t) dt \\
=\;&\int_{t=0}^{t=T} M(t)^\top \left(w(t) + J_z(t)^\top \lambda(t) + \dot{\lambda}(t)\right)  dt + \int_{t=0}^{t=T} J_\theta(t)^\top \lambda(t) dt + M(T)^\top (u - \lambda(T))
\end{align*}
The latter holds for any absolutely continuous function $\lambda$, and in particular, using \cite[Theorem 2, \S 1, Chapter 1]{filippov1988differential}, one can choose $\lambda$ as the unique absolutely continuous solution to 
\begin{align}
				\dot{\lambda}(t) &= -w(t) - J_z(t)^\top \lambda(t),\\
				\lambda(T) &= u.
				\label{eq:backwardODE1}
\end{align}
One has finally
\begin{align*}
				M(T)^\top u + \int_{t=0}^{t=T} M(t)^\top w(t) dt =& \int_{t=0}^{t=T} J_\theta(t)^\top \lambda(t) dt.
\end{align*}
The term $\int_{t=0}^{t=T} J_\theta(t)^\top \lambda(t) dt$ being defined as the sum of two specific elements in $D_L$ and $D_T$ (see Corollary \ref{lem:conservativeIntegralLoss1} and Lemma \ref{th:conservativeIntegralLoss2}), this means that $\int_{t=0}^{t=T} J_\theta(t)^\top \lambda(t) dt\in D_L(z,\theta) + D_T(z,\theta)$, which concludes the proof.
\end{proof}

\subsection{Small step method for optimization}
Recall that we are interested in the following problem, for a fixed $z \in \RR^p$ and $T>0$
\begin{align*}
    \inf_{\theta\in\mathbb{R}^m} \int_0^T \ell(\psi(z,\theta,t)) dt + \ell_T(\psi(z,\theta,T)),
\end{align*}
where the integral cost is $L_I$ and the terminal cost is $L_T$ and their sum is denoted by $L$. Given $\theta \in \RR^m$, Corollary \ref{cor:ajointParam} allows to obtain an element of $D_L(\theta) = D_I(\theta) + D_T(\theta)$. This constitutes relevant first order information. Indeed, for example the results borrowed from \cite[Theorem 1, Corollary 1]{bolte2020conservative} ensure the following properties
\begin{align*}
    \partial^c L(\theta) &\subset \conv\{D_L(\theta)\} , &\text{ for all } \theta \in \RR^m, \\
    D_L(\theta)&= \{ \nabla L(\theta) \} , &\text{ for Lebesgue almost all } \theta \in \RR^m,
\end{align*}
where $\partial^c$ denotes the Clarke subgradient \cite[Chapter 2]{clarke1983optimization}.
As detailed in \cite[Section 6]{bolte2020conservative}, elements of $D_L$ can be used in place of gradients in an optimization context. Given a sequence of positive step sizes $(\alpha_k)_{k\in \NN}$ and $\theta_0 \in \RR^m$, one can iterate the following recursion
\begin{align}
    \theta_{k+1} = \theta_k - \alpha_k g_k \label{eq:algo}\\
    g_k \in D_L(\theta_k).
\end{align}
We insist on the fact that $g_k$ can be obtained for example using Corollary \ref{cor:ajointParam}. Recall that the set of accumulation points of the sequence $(\theta_k)_{k\in \NN}$ is the set of $\bar{\theta}$ such that, for all $r > 0$, the set $\{i \in \NN \, , \|\theta_i- \bar{\theta}\| < r\}$ is infinite. The following result is a consequence of \cite[Theorem 6]{bolte2020long} about bounded sequences of the form \eqref{eq:algo}. This uses a weaker notion of accumulation point to characterize the fact that the sequence is essentially attracted by critical points, that is points which comply with the necessary optimality condition $0 \in \conv\{D_L(\theta)\}$. 

\begin{corollary}
    \label{cor:convergenceSmallStep}
    Assume that $\alpha_k \to 0$ and $\sum_{i=0}^k \alpha_i \to + \infty$ as $k \to +\infty$. Assume furthermore that the sequence $(\theta_k)_{k \in \NN}$ given by \eqref{eq:algo} remains bounded. Then the set
    \begin{align*}
        \Omega = \left\{ \bar{\theta} \in \RR^m,\, \forall r > 0,  \limsup_{N \to \infty} \frac{\sum_{0\leq i \leq N,\,\|\theta_i - \bar{\theta}\| < r} \alpha_i}{\sum_{0\leq i \leq N} \alpha_i}>0\right\}
    \end{align*}
    is non empty and satisfies $\Omega \subset \mathrm{crit}_L$, where $\mathrm{crit}_L$ is the set of $\theta \in \RR^m$ complying with the optimality condition $0 \in \conv\{D_L(\theta)\}$.
\end{corollary}

In the latter corollary, $\Omega$ is termed the ``essential accumulation set'' of the sequence \cite{bolte2020long}. This is a subset of the set of usual accumulation points of the sequence, corresponding to those accumulation points for which the sequence spends a significant amount of time, as measured with respect to the sum of neighboring step sizes. This result illustrates the minimizing behavior of the sequence \eqref{eq:algo}. The result is quite weak, but provides a solid ground regarding the use of the proposed conservative adjoint method for finite dimensional optimization under ODE constraints. At this level of generality, stronger assumptions, such as Sard type conditions related to the loss $L$, would be required to obtain stronger statements \cite{rios2020examples}. This will be a topic of future research.

\section{Conclusion}

\label{sec_conclusion}

In this article, we have proved that flows of ODEs expressed with vector fields that are path differentiable are also path differentiable. The proof of this results stands on the fact that the set valued mapping obtained from the variational inclusion is a conservative Jacobian. This allows to develop a conservative calculus for integral costs, similar as one would have in the smooth case. This culminates with a conservative version of the adjoint method to propagate derivatives backward and obtain gradients of integral costs, with a considerable reduction of the size of the differential inclusion to be solved. A consequence of these results is the fact small step methods of gradient type for minimizing integral costs are attracted by solutions of a certain optimality condition for such problems with path differentiable data. 

The developments provided in this work could be extended to parametric partial differential equations (PDEs), this was actually one of the original motivations for the proposed developments. The question of path differentiability of PDEs could be considered under regularity assumption by exhibiting a conservative Jacobian in a way similar to what is proposed for ODE flows. One should probably start with specific sub-classes of PDEs, for instance hyperbolic ones \cite{dafermos2005hyperbolic}.

\section*{Acknowledgements}
Edouard Pauwels acknowledges the support of AI Interdisciplinary Institute ANITI funding, through
the French ``Investing for the Future - PIA3'' program under the Grant agreement ANR-19-PI3A0004, Air Force Office of Scientific Research, Air Force Material Command, USAF, under grant numbers FA9550-19-1-7026, FA9550-18-1-0226, and ANR MaSDOL 19-CE23-0017-01. 

\appendix 

\section{Technical results}

\label{sec_appendix}

This appendix is devoted to the statement and the proof of some crucial results for our analysis.

The following result is about the Borel measurability of partial derivatives. Its proof is inspired by \cite[Theorem 3.2]{evans2015measure}. 

\begin{lemma}[Measurability of partial derivatives]
\label{lemma-direct}
Consider a function $G:(x,y)\in \mathbb{R}^{n}\times\mathbb{R}\rightarrow G(x,y)\in \mathbb{R}^m$. Suppose that, for all $x\in\mathbb{R}^n$, $y\mapsto G(x,y)$ is absolutely continuous, and for all $y$, $x \mapsto G(x,y)$ is Borel measurable. Then, the function $(x,y) \mapsto \frac{\partial}{\partial y}G(x,y)$ defined on a set of full Lebesgue measure and is measurable. Futhermore, for all $x$, the function $y\mapsto \frac{\partial}{\partial y} G(x,y)$ exists for almost all $t$. 
\end{lemma}

\begin{proof}

As a Carath\'eodory function, $G$ is jointly Borel measurable \cite[Lemma 4.51]{aliprantis2005infinite} and therefore it is Lebesgue measurable.
Consider the functions:
\begin{equation}
\label{direct-deriv1}
G_y^u(x,y):= \lim_{h\rightarrow 0} \sup \frac{G(x,y+h) - G(x,y)}{h}
\end{equation}
and
\begin{equation}
\label{direct-deriv2}
G_y^l(x,y):= \lim_{h\rightarrow 0} \inf \frac{G(x,y+h) - G(x,y)}{h},
\end{equation}
with $h\in \mathbb{R}$.

Using the continuity of $G$ in its second argument, one has the equivalent definition \eqref{direct-deriv1} by
\begin{equation}
  G^u_y(x,y) = \lim_{k\rightarrow +\infty} \sup_{0< |h|< \frac{1}{k},\, h \in \mathbb{Q}} \frac{G(x,y+h) - G(x,y)}{h},
  \label{eq:direct-derivTemp1}
\end{equation}
For all $k \geq 1$, the set $\lbrace h\in\mathbb{Q}\mid 0< |h|<  \frac{1}{k}\rbrace$ is countable and therefore, using Lemma \cite[Corollary 7, Section 18.1]{royden2010real}, the supremum in \eqref{eq:direct-derivTemp1} is Lebesgue measurable as the countable supremum of measurable functions. This implies that the sequence 
$$\left\{\sup_{0< |h|< \frac{1}{k}} \frac{G(x,y+h) - G(x,y)}{h}\right\}_{k\in\mathbb{N}}$$ 
is a bounded, decreasing sequence of  measurable functions and it has therefore a pointwise limit everywhere. Using \cite[Chapter 18, Corollary 7]{royden2010real}, the pointwise limit of measurable functions is a measurable function. This implies that $G_y^u$ defined in \eqref{eq:direct-derivTemp1} or equivalently in \eqref{direct-deriv1} is Lebesgue measurable. By a similar argument, one can show that \eqref{direct-deriv2} is also measurable.  

It remains to prove that the function $(x,y)\mapsto \frac{\partial}{\partial y} G(x,y)$ exists Lebesgue almost everywhere. To do so, consider the Lebesgue measurable set 
$$
A:= \left\{ (x,y)\in\mathbb{R}^{n+1}, G^u_y(x,y)= G^l_y(x,y),\, G^l_y(x,y) \neq \pm \infty\right\}
$$
Its complement $A^c$ is the subset of $\mathbb{R}^{n+1}$ where $G^u_y(x,y) \neq G^l_y(x,y)$ or $G^l_y(x,y) = \pm \infty$. By Fubini's theorem \cite[Theorem 16 Section 20.2]{royden2010real},
\begin{align*}
    \int_{(x,y)} \mathbb{I}_A(x,y)dx dy = \int_{x} \left(\int_y \mathbb{I}_A(x,y) dy\right) dx = 0.
\end{align*}
where $\mathbb{I}_A$ is the function such that $\mathbb{I}_A(x,y) = 1$ if $(x,y) \in A$ and $0$ otherwise. By Lebesgue integration theorem \cite[Theorem 10, Section 6.5]{royden2010real}, the inner integral is zero because of the absolute continuity in $y$ for fixed $x$.
This shows that the function $(x,y)\mapsto \frac{\partial}{\partial y} G(x,y)$  exists for Lebesgue almost all $(x,y) \in \mathbb{R}^{n+1}$. Furthermore, up to an arbitrary measurable choice outside of its domain of definition, it is a Lebesgue measurable function. This concludes the proof of Lemma \ref{lemma-direct}. 
\end{proof}

We also state a useful result which states that every lower semicontinuous functions are Borel measurable.

\begin{lemma}
\label{lem:lower-measurable}
Given $X$ a metric space, let $f:X\rightarrow \mathbb{R}$ be a lower semicontinuous real function. Then, it is Borel measurable.
\end{lemma}

\begin{proof}
The function $f$ being lower semicontinuous, the set $\lbrace (x,c)\in X\times \mathbb{R}\mid c\geq f(x) \rbrace$ is closed. It is in particular a Borel set, which implies that $f$ is Borel measurable. This concludes the proof.\end{proof}

The following result concerns derivatives of integrals. More precisely, it states and proves that the operators derivatives and integrals can be permutated. This result is closely related to the well-known Leibniz rule, but it concerns in our case absolutely continuous functions. It can be seen then as a general Leibniz rule. 

\begin{lemma}[General Leibniz rule]
\label{lem:leibniz}
Consider a Lipschitz function $F: \mathbb{R}^p \times X \rightarrow \mathbb{R}^p$ where $X$ is a bounded interval of $\mathbb{R}$. Consider furthermore an absolutely continuous function $\gamma: [0,1]\rightarrow \mathbb{R}^p $. Then, $r \mapsto \int_X F(\gamma(r),s) ds$ is absolutely continuous and for a.e $r\in [0,1]$ : 
\begin{equation}
\frac{d}{dr} \int_X F(\gamma(r),s) ds = \int_X \frac{\partial}{\partial r} F(\gamma(r),s) ds. 
\label{eq:leibnizAbsCont}
\end{equation}
\end{lemma}

\begin{proof}
Since $F$ is Lipchitz continuous and $\gamma$ is absolutely continuous, then, for all $s\in X$, the function $r\mapsto F(\gamma(r),s)$ is absolutely continuous as the composition of a Lipschitz function with an absolutely continuous function. In particular, it is differentiable for a.e. $r\in [0,1]$.

Furthermore, since the function $(r,s)\mapsto F(\gamma(r),s)$ is continuous, one can prove that, due to Lemma \ref{lemma-direct}, the function:
\begin{equation}
\label{eq:partial-jointly}
(r,s)\mapsto \frac{\partial }{\partial r} F(\gamma(r),s)
\end{equation}
is well defined for all $s\in X$ and for a.e. $r\in [0,1]$. It is also jointly measurable in $(r,s)$ (up to arbitrary values outside of its full measure domain of definition). Denoting by $L$ a Lipschitz constant of $F$, we have we have for all $s$ and almost all $r$
\begin{equation}
\label{eq:partial-jointly-Modulus}
\left\|\frac{\partial }{\partial r} F(\gamma(r),s)\right\| \leq L \|\dot{\gamma}(r)\| 
\end{equation}
by definition of the derivative.

Then, using again the fact that, for all $s\in X$, $r\mapsto F(\gamma(r),s)$ is absolutely continuous, then one has, for all $s\in X$ and $r \in [0,1]$,
\begin{equation}
F(\gamma(r),s) - F(\gamma(0),s) = \int_0^r \frac{\partial}{\partial q} F(\gamma(q),s) dq. 
\end{equation}
Integrating the previous equation over the domain $X$ leads to
\begin{equation*}
\int_X [F(\gamma(r),s) - F(\gamma(0),s)] ds = \int_X \int_0^r \frac{\partial}{\partial q} F(\gamma(q),s) dq ds. 
\end{equation*}

Since the function given by \eqref{eq:partial-jointly} is jointly measurable, Fubini's theorem \cite[Theorem 11.27]{aliprantis2005infinite} applies and one has for all $r \in [0,1]$
\begin{equation*}
\int_X [F(\gamma(r),s) - F(\gamma(0),s)] ds = \int_0^r \int_X \frac{\partial}{\partial q} F(\gamma(q),s) ds dq. 
\end{equation*}
This proves the desired result because \eqref{eq:leibnizAbsCont} is a consequence of Lebesgue differentiation theorem \cite[Section 6.5]{royden2010real}, for all $q$,
\begin{align*}
    \left\|\int_X \frac{\partial}{\partial q} F(\gamma(q),s)  ds\right\| \leq L \|\dot{\gamma}(q)\| \times \int_X ds
\end{align*} 
where right hand side is integrable because $X$ is a bounded interval and $\gamma$ is absolutely continuous (its derivative is integrable). This concludes the proof.\end{proof}

Next, we provide a generalization of the Gr\"onwall's inequality for absolutely continuous functions. The proof is inspired by the proof of \cite[Theorem 2]{filippov1988differential}.

\begin{lemma}[Gr\"onwall's Lemma for absolutely continuous functions]
\label{lem:gronwall}
    Let $K$ be a constant and $f:\mathbb{R}_+\rightarrow \mathbb{R}_+$ be absolutely continuous on $[0,T]$ such that, for a.e. $t\in [0,T]$:
    \begin{equation}
    \label{gronwal-inequality}
    \frac{d}{dt} f(t) \leq K f(t).
    \end{equation}
    Then, $f(t) \leq \exp(Kt) f(0)$ for all $t\in [0,T]$.
\end{lemma}

\begin{proof}
From the inequality, one has, for a.e. $t\in [0,T]$, 
$$
\left(\frac{d}{dt} f(t) - K f(t)  \right) \exp(-K t) = \frac{d}{dt}(f(t) \exp(-Kt)) \leq 0.
$$
The function $t\mapsto f(t) \exp(-Kt)$ is absolutely continuous as a product of absolutely continuous functions with bounded domain. One deduces therefore that, for all $t\in [0,T]$, $f(t)\exp(-Kt)f(t) - f(0) \leq 0$. Then, for all $t\in [0,T]$, $f(t)\leq f(0) \exp(Kt)$, which concludes the proof of the lemma.
\end{proof}

Associated to this inequality, one may deduce a Gr\"onwall inequality for integrable functions, as stated in the following lemma.

\begin{lemma}[Gr\"onwall's inequality for integrable functions]
\label{lem:gronwall_int}
Let $K$ be a positive constant and $f \colon \RR_+ \to \RR_+$ be integrable on $[0,T]$, such that for all $t \in [0,T]$
\begin{align*}
    f(t) \leq K \int_0^t f(s) ds.
\end{align*}
Then $f(t) = 0$ for all $t \in [0,T]$.
\end{lemma}
\begin{proof}
Let $G\colon t\mapsto \int_0^t f(s) ds$. This function is absolutely continuous, nonnegative (since $f$ is nonnegative), and for almost all $t\in [0,T]$, one has $ \frac{d}{dt} G(t) \leq K G(t)$. Then, one can apply Lemma \ref{lem:gronwall} and deduce that, for all $t\in [0,T]$, $G(t)\leq \exp(Kt) G(0)$. Therefore, since $G(0) = 0$ one has $G(t) = 0$ for all $t\geq 0$, which implies that $f(t) = 0$ since $0 \leq f(t) \leq K G(t)$. This concludes the proof of the lemma.  \end{proof}

Finally, we provide a density result showing that for a set valued map $D$ to be conservative for a given function $f$, it is sufficient to prove a conservativity relation for each element of a Castaing representation of $D$.

\begin{lemma}[Measurable selections and density]
\label{lem:dense}
Consider a set-valued map $D:x\in \mathbb{R}^p \rightrightarrows D(x)\in \mathbb{R}^p$ that is locally bounded and has non empty values and a closed graph. Consider a sequence of measurable selections in $D$, denoted by $\lbrace M_i\rbrace_{i\in\mathbb{N}}$, such that, for all $x$
$$
D(x) = \overline{\lbrace M_i(x)\rbrace_{i\in\mathbb{N}}}.
$$
Then, for a given function $f;\: \mathbb{R}^p\rightarrow \mathbb{R}$, if, for all $i\in\mathbb{N}$ and for all absolutely continuous path $\gamma:\: [0,1]\rightarrow \mathbb{R}^p$, one has, for all $t\in [0,1]$
\begin{equation}
\label{eq:lem:integral}
f(\gamma(t)) - f(\gamma(0)) = \int_0^t M_i(\gamma(s))\dot{\gamma}(s) ds,\:
\end{equation}
then $D$ is a conservative Jacobian for $f$.
\end{lemma}

\begin{proof} Fix an absolutely continuous path $\gamma$. By assumption, since $D$ is locally bounded, we have that $f$ is locally Lipschitz and therefore $f \circ \gamma$ is absolutely continous.

For each $i \in \NN$, by Lebesgue integration theorem,
\begin{equation}
\label{eq:lem:diff}
    \frac{d}{dt} f(\gamma(t)) =  M_i(\gamma(t)) \dot{\gamma}(t),\: \text{ for a.e. $t\in [0,1]$.}
\end{equation}
For $i \in \NN$, consider the set defined by 
$$E_i:=\lbrace t\in [0,1]\text{ such that } \eqref{eq:lem:diff}\text{ holds}\rbrace \subset [0,1].$$ 
The set $E_i$ has full measure for each $i \in \NN$.
Then, we define $E:=\cap_{i\in \mathbb{N}} E_i$. Since $E$ is a countable intersection of full measure sets, its complement $E^c$ has a Lebesgue measure zero.

Therefore, for all $t\in E$, \eqref{eq:lem:diff} holds for any $i\in\mathbb{N}$. Since $\overline{\left\{M_i(\gamma(t))\right\}_{i\in\mathbb{N}}}=D(\gamma(t))$, we have
\begin{equation}
\frac{d}{dt} f(\gamma(t)) = W\gamma(t),\: \forall W\in D(\gamma(t)),
\end{equation}
for all $t \in E$. 
Since $E$ has full measure and $\gamma$ is an arbitrary absolutely continuous path, this shows that $D$ is a conservative Jacobian for $f$ and then concludes the proof of the Lemma. 
\end{proof}

Another important result is the integration by parts formula for absolutely continuous functions, that is stated as follows. 
\begin{lemma}[Integration by parts]
\label{lem:ipp}
Consider two absolutely continuous functions $f,g \colon [0,T]\rightarrow \mathbb{R}^p$. Then, the following integration by parts formula holds
\begin{equation}
\int_0^T f(t) \dot{g}(t) dt  = f(T)g(T) - f(0)g(0) - \int_0^T \dot{f}(t) g(t) dt. 
\end{equation}
\end{lemma}
\begin{proof}
Since $f$ and $g$ are absolutely continuous with bounded domains, then the product $fg$ is absolutely continuous. By definition of absolutely continuous functions, one has
$$
\int_0^T \frac{d}{dt}(fg)(t) dt = f(T)g(T) - f(0)g(0).
$$
Noticing that, for a.e. $t\in [0,T]$, $\frac{d}{dt}(fg)(t) = \dot{f}(t)g(t) + f(t)\dot{g}(t)$, one obtains the desired result. 
\end{proof}

We also state and prove a result stating that the space of Lipschitz functions whose domain is a given bounded interval is $\sigma$-compact.
\begin{lemma}
    The space $\mathcal{L}$ of $L$ Lipschitz functions from $[0,T]$ to $\RR^q$, equipped with the suppremum norm $\|\cdot\|_\infty$, is $\sigma$-compact and Hausdorff.
    \label{lem:lipschitzSigmaCompact}
\end{lemma}
\begin{proof}
    For any $\bar{u} \in \mathcal{L}$ and $\bar{v} \in \mathcal{L}$, with $\bar{u} \neq \bar{v}$ we have $\|\bar{u} - \bar{v}\|_\infty > 0$ and therefore $U = \{u \in \mathcal{L},\, \|u - \bar{u}\|_\infty <  \|\bar{u} - \bar{v}\|_\infty / 4 \}$ and $V = \{v \in \mathcal{L},\, \|v - \bar{v}\|_\infty <  \|\bar{u} - \bar{v}\|_\infty / 4 \}$ form two disjoint neighborhoods of $\bar{u}$ and $\bar{v}$ and we have Hausdorff separation condition. Furthermore, denoting by $\mathcal{B}_\infty (s)$ the ball centered at $0$ of radius $s>0$ in $\mathcal{L}$, we have $\mathcal{L} = \cup_{i \in \NN} \mathcal{L} \cap \mathcal{B}_\infty (i)$. Since all the functions of this set are $L$-Lipschitz and bounded, each element of the union is sequentially compact using Arzel\`a-Ascoli Theorem \cite[Theorem 4.25]{brezis2010functional}. Thus $\mathcal{L}$ is a countable union of compact subspaces, meaning that the space $\mathcal{L}$ is $\sigma$-compact.
\end{proof}

We now state and prove a result dealing with the projection of conservative Jacobians. 
\begin{lemma}
\label{lem:projection}
Let $G(x,y):\mathbb{R}^{p+m}\rightarrow \mathbb{R}^n$ be a path-differentiable function whose conservative Jacobian is denoted by $J_G:\mathbb{R}^{p+m} \rightrightarrows \mathbb{R}^{n\times(p+m)}$. Consider
\[
\Pi_y J_G(x,y):= \lbrace M_2\in \mathbb{R}^{n\times m}, \exists M_1\in\mathbb{R}^{n\times p}, (M_1,M_2)\in J_G(x,y)\rbrace. 
\]
Then, for all $x\in\mathbb{R}^p$, $\Pi_y J_G(x,y)$ is conservative for the function $y\mapsto F(x,y)$.
\end{lemma}

\begin{proof}
Consider an absolute continuous function $\gamma: [0,1]\rightarrow \mathbb{R}^m$. Then, the function
\begin{align*}
\tilde \gamma: [0,1]&\rightarrow \mathbb{R}^{p+m}\\
t&\mapsto \begin{pmatrix}
    0 \\ \gamma(t)
\end{pmatrix}
\end{align*}
is also an absolute continuous function. Then, for every $x\in \RR^p$, it is clear that
$$
J_F(\tilde{\gamma}(t))\dot{\tilde{\gamma}}(t) = \Pi_y J_F(x,\gamma(t)) \dot{\gamma}(t).
$$
From this identity, and by definition of the conservativity, one can show that, for every $x\in \mathbb{R}^p$, $\Pi_y J_F(x,\gamma(t))$ is conservative for $y\mapsto G(x,y)$, which concludes the proof.\end{proof}

\end{document}